\documentclass[final]{siamart0516}
\usepackage{lipsum}
\usepackage{amsfonts}
\usepackage{graphicx}
\usepackage{epstopdf}
\usepackage{algorithmic}
\usepackage{comment}    

\usepackage{bbm}

\ifpdf
\DeclareGraphicsExtensions{.eps,.pdf,.png,.jpg}
\else
\DeclareGraphicsExtensions{.eps}
\fi
\newcommand{\TheTitle}{Experimental Design for Non-Parametric Correction of Misspecified Dynamical Models}
\newcommand{\TheShortTitle}{Experimental Design for Correcting Dynamical Models}
\newcommand{\TheAuthors}{G. Shulkind, L. Horesh, and H. Avron}
\headers{\TheShortTitle}{\TheAuthors}
\title{{\TheTitle}\thanks{Submitted to the editors on May 1, 2017.
\funding{G. Shulkind work was partly done while at IBM TJ Watson Research center; H. Avron acknowledges support by the XDATA program of the Defense Advanced Research
Projects Agency (DARPA), administered through Air Force Research Laboratory contract FA8750-
12-C-0323 and of an IBM Faculty Award.}}}

\author{
	Gal Shulkind\thanks{Department of Electrical Engineering and Computer Science, Massachusetts Institute of Technology, Cambridge, MA, USA,  \email{shulkind@mit.edu}.}
	\and
	Lior Horesh\thanks{Quantum Computing, IBM TJ Watson Research Center, Yorktown Heights, NY, USA, (\email{lhoresh@us.ibm.com},
		\email{ibm.biz/lhoresh}).}
	\and
	Haim Avron\thanks{Department of Applied Mathematics, Tel Aviv University, Israel, \email{haimav@post.tau.ac.il}, \url{http://www.math.tau.ac.il/\string~haimav}.}
} 

\usepackage{amsopn}

\ifpdf
\hypersetup{
	pdftitle={\TheTitle},
	pdfauthor={\TheAuthors}
}
\fi

\begin{document}
	\maketitle
	\begin{abstract}
We consider a class of misspecified dynamical models where the governing term is only approximately known. Under the assumption that observations of the system's evolution are accessible for various initial conditions, our goal is to infer a non-parametric correction
to the misspecified driving term such as to faithfully represent the system dynamics and devise system evolution predictions for unobserved initial conditions.

We model the unknown correction term as a Gaussian Process and analyze the problem of efficient experimental design to find an optimal correction term under constraints such as a limited experimental budget.
We suggest a novel formulation for experimental design for this Gaussian Process and show that approximately optimal (up to a constant factor) designs may be efficiently derived by utilizing results from the literature on submodular optimization.
Our numerical experiments exemplify the effectiveness of these techniques.
\end{abstract} 
	\begin{keywords}
		Model Misspecification, Dynamical Systems, Experimental Design, Submodularity, Gaussian Processes.
	\end{keywords}
	
	\begin{AMS}
		62K05, 37M05, 62G08, 68T05 
	\end{AMS}

	\section{Introduction}
The evolution of a wide variety of dynamical systems can be described by mathematical models which embody differential equations \cite{gear1971numerical}. Such dynamical models are employed ubiquitously for description, prediction, and decision-making under-uncertainty. While the primary role of a mathematical model is to provide a consistent link between the input and output of a system or phenomenon under investigation, multiple considerations are at play when designing a model, involving a series of choices which influence its complexity and realism. These choices represent trade-offs between different competing objectives including model accuracy, robustness, functional complexity, scalability, computational complexity and interpretability. 

Acknowledging that ``essentially all models are wrong'' \cite{box1978bayesian}, a fundamental question is ``what is the desired level of fidelity required by the model?''.  This question cannot be answered in isolation, as often the required level of fidelity cannot be assessed directly, especially when mathematical models are embedded within an end-goal optimization or a decision process. In such circumstances, prominent modeling errors creep into simulation-based insights, ramifications of which could be inaccurate state descriptions, unstable model inferences, or erroneous control output, designs or decisions.  As a guiding principle, uncertainty propagation should be holistically accounted for to ensure that we invest sufficient, yet non-redundant effort into each stage in the information flow value chain \cite{lieberman2013goal}. 


For a broad range of applications, especially in the control space, the modeler's ability to devise a model of higher fidelity than required for the application, enables deliberate compromise of an adequate model in exchange for reduced model complexity, which entails more economic computation. Examples would be model reduction approaches such as proper orthogonal decomposition and discrete empirical interpolation method \cite{Peherstorfer201521,chaturantabut2010nonlinear} as well as multi-fidelity modeling, where routine computation of high fidelity models may be intractable for the underlying task, necessitating the use of low fidelity proxies \cite{alexandrov2001approximation,leary2003knowledge,antoulas2005approximation,rajnarayan2009trading,goh2012prediction,march2012multidelity}. 

Conversely, for a variety of applications the goal is to improve the accuracy of an inadequate model such as to comply with minimum desired fidelity requirements \cite{box2011bayesian,kutoyants2012identification,shalizi2009dynamics}.
One such scenario is in situations where domain knowledge or first principles approach are employed for the description of a  system of interest. Conventionally, in such settings, a human expert derives  
a low-complexity approximate model for description of  the dynamics \cite{Peherstorfer201521,he1997lattice,haber2000fast,virieux2009overview,kutoyants2012identification,soize2008probabilistic,shalizi2009dynamics}. 
Real-world 
phenomena often involve additional, weak effects, that are not accounted for in typical expert derived models, rendering such models inadequate representations of reality. For example, in designing electrical circuits, ideal linear models are often assumed for circuit components such as resistors, capacitors and inductors, however, available components tend to exhibit weak but complicated non-linear characteristics not accounted for by the approximate models \cite{chase2005modeling,martinez2005transformer,fatoorehchi2015analytical}. Another example is in deriving models for flow systems, where idealized models may be assumed for the medium and its boundaries, neglecting weak nonlinear phenomena and deteriorating the fidelity of the resulting models \cite{doi:10.1137/S003614290038296X,Veneziani2005FlowRD}.
In other settings multi-physics coupling effect may not be readily apparent or properly characterized by the modeler \cite{DONEA1982689,doi:10.1063/1.322296,cox2006two}
or multi-dimensional model construction may harness approximated models to accommodate computational limitations \cite{zunino2004multidimensional,blanco2010assessing,tinsley2006multiscale}. Other common sources of model misspecification may be related to simplified representation of the domain geometry 
\cite{tizzard2005generating}, isotropic modeling of anisotropic medium \cite{thomsen1986weak,wei1995comparative,abascal2008use} and so on. Additionally, the model may be misspecified due to either conscious or non-conscious choices made regarding the numerical solution of the underlying system: immature truncation of infinite expansions,
round-off errors, approximate solutions  of linear or non-linear terms, etc \cite{helms1962truncation,jagerman1966bounds,weideman1994computation,trefethen2013approximation}.  


In lieu of deriving an approximate model which 
may confer an inadequate representation of the system's dynamic, a common alternative is to take a completely agnostic, data-driven approach 
and apply
either parametric 
or non-parametric techniques to  learn the dynamics purely based on empirical data collected from the system \cite{hastie2009elements,mcgoff2015statistical,Peherstorfer201521,mcgoff2015statistical}. 
However, 
such an agnostic approach may entail models with  several potential shortcomings \cite{horesh2015tnobel,lam2016copper}:
\begin{itemize}
	\item Failure to utilize crucial 
prior knowledge regarding the system and its dynamics
\item Relies on the availability of a large set of training examples to derive complex models of sufficient fidelity 
\item Poor generalization performance for out-of-sample instances 
\item Limited means for interpretability due to the agnosticism of the underlying functional form 
\end{itemize}

In this study we explore a third approach of symbiotizing 
 these two information sources effectively: on the one hand a crude misspecified system model as derived based on domain knowledge, and on the other hand empirical measurements and data to complement the misspecified model. Our goal is to learn a 
generalized, non-parametric %
representation for the system dynamics, based on the approximate model and the empirical data. We focus on understanding how this learning process can be performed efficiently, with only a limited budget for experiments to probe the system and collect empirical data points \cite{haber2008numerical,horesh2010optimal,tenorio2013experimental}. Specifically, we explore the role of the initial approximate model in guiding the design of experiments for collection of empirical data that best informs the model correction objective.

The choice between parametric and non-parametric representation of the correction term depends upon the knowledge available to the modeler and in particular, how well the functional class of the correction model is fully understood. When an explicit parametric  representation of the correction term is known, the model correction problem reduces to meta-level parametric estimation problem  \cite{blight1975bayesian,shalizi2009dynamics,hao2014nuclear}. 
In more general settings, where various functional representations (or combinations of which) from a function class may comprise the correction term, a non-parametric approach may be more appropriate \cite{kennedy2001bayesian}. 
The non-parametric option requires weaker, implicit assumptions regarding the desired correction, and therefore applicable to a broader class of problems. We have chosen to focus on that case in the current study.  

Non-parametric formulations offer great versatility in defining non-linear functional representations\cite{aronszajn1950theory,micchelli2006universal,berlinet2011reproducing}, as well as scalable means for their learning \cite{avron2016high,avron2016krr}. Thus, in this study, we take the approach proposed by Kennedy and O'Hagen \cite{kennedy2001bayesian} and articulate the misspecified function in terms of Gaussian Processes. 
Following representation of the misspecified term by GPs, and relying upon the representer theorem \cite{scholkopf2001generalized}, the model correction learning problem is fully determined. 

The end-goal, however, is to accelerate the learning curve, and infer the correction, with minimal number of observations. In this study, we consider a Bayesian {\it D-optimal} experimental design \cite{fedorov1972theory,pukelsheim1993optimal} where we maximize the information gain through  collection of informative data. Following the work of Krause and Golovin \cite{krause2012submodular} we prove that Mutual Information can be regarded as a monotonic sub-modular function in our settings. Based upon this observation, we gain access to the wealth of optimization machinery available for optimization of sub-modular set functions \cite{minoux1978accelerated,nemhauser1978analysis,calinescu2007maximizing,oxley2006matroid}, and thereby provide solid performance guarantees

\paragraph{Study structure} The paper is structured as follows. In Section \ref{sec:problem_formulation} we formulate the problem of misspecified dynamical systems. In Section \ref{sec:Correction_term} we briefly review the Gaussian Process (GP) formulation and its application in concise representation of correction terms. Section \ref{sec:experimental_design} studies experimental design in the context of misspecified models. 
In Section \ref{sec:numerical_experiments}  we perform numerical experiments to validate and demonstrate our results, and in Section \ref{sec:conclusions-and-future-extensions} we conclude the work. 
	\section{Problem Formulation} \label{sec:problem_formulation}
The behavior of a broad variety of dynamical models can be described by Ordinary Differential Equations (ODE) \cite{gear1971numerical}, thus we consider a misspecified system of first order ODEs (higher order systems may be converted into first order form by the usual techniques):
\begin{align} \label{eq:mis_specified_model}
\frac{d}{dt}\textbf{y}(t)=\textbf{G}(\textbf{y}(t))+\textbf{F}(\textbf{y}(t))
\end{align}
with $t$ time, $\textbf{y}(t)=[y_1(t),\ldots,y_d(t)]^{\top}$ a vector signal of interest, and $\textbf{F}(\textbf{y}(t))=[F_1(\textbf{y}(t)),\ldots, F_d(\textbf{y}(t))]^{\top}, \textbf{G}(\textbf{y}(t))=[G_1(\textbf{y}(t)),\ldots, G_d(\textbf{y}(t))]^{\top}$ vector valued functions, $\textbf{F}(\cdot), \textbf{G}(\cdot):\mathbb{R}^d\rightarrow\mathbb{R}^d$ governing the system dynamics.

We are interested in settings where the temporal evolution of $\textbf{y}(t)$ is dominated by the component $\textbf{G}(\cdot)$, whereas the correction term $\textbf{F}(\cdot)$ is assumed to have only a small effect over short time spans. Concretely, define the auxiliary system
\begin{align} \label{eq:approx_sys_model}
\frac{d}{dt}\textbf{y}_G(t)=\textbf{G}(\textbf{y}_G(t))
\end{align}
then our interest is in the regime where, initialized in the same state $\textbf{y}(0)=\textbf{y}_G(0)$ the two systems track each other closely over some prescribed time span $t_f$. Specifically, we assume that for all  $t\in [0,t_f]$ we have $\|y(t){-}y_G(t)\|{\leq}\delta$. One way of ensuring this is requiring $\|\textbf{F}(\cdot)\|\ll\|\textbf{G}(\cdot)\|$ over some domain $\mathcal{D}\subseteq\mathbb{R}^d$ , and short enough time spans $t_f$. The model \eqref{eq:mis_specified_model} is misspecified in the sense that $\textbf{G}(\cdot)$ is assumed known, whereas the small additive correction function $\textbf{F}(\cdot)$ is not available. Situations like this may arise, e.g. when we have at our disposal some approximation $\textbf{G}(\cdot)$ to a system of interest, perhaps derived via expert domain knowledge, which does not fully capture the true dynamics driving the system.

In this paper, our goal is to utilize system evolution paths $\textbf{y}(t)$ as observed in experiments to learn a representation for the correction term $\textbf{F}(\cdot)$. The resulting 'corrected' model allows making accurate predictions about the system evolution. We specifically focus on designing efficient experiments that facilitate rapid learning of the correction term under a limited experimental budget.
\subsection{Initial Conditions}
We consider applications where we are at liberty to perform a limited number of at most $K$ experiments to facilitate learning the correction term $\textbf{F}(\cdot)$. The $\textit{k}^{\rm{th}}$ experiment entails preparing the system at some fixed initial conditions at time zero $\textbf{y}^{(k)}(0)\in\mathcal{Y}$ and observing its subsequent evolution $\textbf{y}^{(k)}(t)$ for $t>0$ as determined by the (not fully known) model
\eqref{eq:mis_specified_model}  (fixing initial conditions at time $t=t_0$ the output of the first order ODE system \eqref{eq:mis_specified_model} is determined for all $t>t_0$ \cite{gear1971numerical}). We take the set $\mathcal{Y}$ to be a finite collection of possible experimental conditions that we may choose to start the system from. It may be a finely discretized grid over a continuous region of accessible initial conditions, e.g. expressing power constraints $\mathcal{Y}\subset\left\lbrace\textbf{y} : \|\textbf{y}\|_2^2\leq P\right\rbrace$, or otherwise meeting an application specific set of restrictions.

Let $\mathcal{Y}_0\subseteq \mathcal{Y}, \vert \mathcal{Y}_0\vert\leq K$ be the set of selected initial conditions that seed the $K$ experiments.
Informative prescription of 
$\mathcal{Y}_0$ is a primary concern in this study, as in many practical scenarios experiments are costly and it is important to design them carefully in order to extract as much information as possible from the limited set of measurements.

\subsection{An Observation Model} \label{sec:obs_model}
The empirical evolution data $\textbf{y}^{(k)}(t),\, k=1,\ldots,K$ allows us to probe the system dynamics and learn a representation for the correction term. To set the framework we specify a discrete and noisy observation model.

In this work we assume readings are collected on a discrete time grid. As the $\textit{k}^{\rm{th}}$ experiment unfolds the system evolves from the initial state $\textbf{y}^{(k)}(0)\in \mathcal{Y}_0$ according to $\textbf{y}^{(k)}(t)$ and we gain access to $T$ temporal observations on a discrete time grid $t\in\mathcal{T}=\lbrace t_1,\ldots ,t_T\rbrace$. Let $\mathcal{Y}_m\equiv\left\lbrace
\textbf{y} \vert \exists k,i \enspace \text{s.t.}\enspace \textbf{y}=\textbf{y}^{(k)}(t_i)\right\rbrace$ be the set of size $\tilde{K}\equiv|\mathcal{Y}_m|=KT$ of system states recorded during the $K$ experiments seeded by states in $\mathcal{Y}_0$.

For a given trajectory $\textbf{y}(t)$, it is apparent from \eqref{eq:mis_specified_model} that the correction term $\textbf{F}(\cdot)$ can be evaluated at points along the path via
\begin{align} \label{eq:data_collection}
\textbf{F}(\textbf{y}(t))=\frac{d}{dt}\textbf{y}(t)-\textbf{G}(\textbf{y}(t))
\end{align}

With oracle access to the derivative $\frac{d}{dt}\textbf{y}(t)$ we could attain point samples of $\textbf{F}(\textbf{y})$ for all $\textbf{y}\in \mathcal{Y}_m$ through \eqref{eq:data_collection} as $\textbf{G}(\textbf{y})$ is assumed known. However, with only discrete samples on the trajectory we do not have access to $\frac{d}{dt}\textbf{y}(t)$. Instead, we assume access to noisy derivative estimates $\frac{d}{dt}\tilde{\textbf{y}}^{(k)}(t_i)$  (e.g. by simple numerical differences, or more advanced techniques performing smoothing over the trajectory \cite{conn2009introduction,trefethen2013approximation} 
 according to:
\begin{align}
\frac{d}{dt}\tilde{\textbf{y}}^{(k)}(t_i)=\frac{d}{dt}\textbf{y}^{(k)}(t_i)+\boldsymbol{\epsilon}^{k,i} \qquad k=1,\ldots,K,\quad i=1,\ldots,T
\end{align}
with $\boldsymbol{\epsilon}^{k,i}\sim\mathcal{N}(0,\boldsymbol{\Sigma}_{\epsilon})$ i.i.d. Gaussian noise. We form noisy estimates for the correction term $\tilde{\textbf{F}}(\textbf{y}^{(k)}(t_i))$ by substituting:
\begin{align} \label{eq:noisy_meas}
	\tilde{\textbf{F}}(\textbf{y}^{(k)}(t_i))&\equiv\frac{d}{dt}\tilde{\textbf{y}}^{(k)}(t_i)-G(\textbf{y}^{(k)}(t_i))=\frac{d}{dt}\textbf{y}^{(k)}(t_i)-G(\textbf{y}^{(k)}(t_i))+\boldsymbol{\epsilon}^{k,i}\notag\\
	&={\textbf{F}}(\textbf{y}^{(k)}(t_i))+\boldsymbol{\epsilon}^{k,i}
\end{align}

For the sequel, we sometimes ease notations by writing $\textbf{f}^j\equiv\textbf{F}(\textbf{y}^j)$, and $\tilde{\textbf{f}}^j\equiv\tilde{\textbf{F}}(\textbf{y}^j)$ for the noisy readings $\textbf{y}^j\in \mathcal{Y}_m, \quad j=1,\ldots,\tilde{K}$. In these symbols the noisy measurement model \eqref{eq:noisy_meas} reads:
\begin{align} \label{eq:noisy_measurements}
	\tilde {\textbf{f}}^j=\textbf{F}(\textbf{y}^j)+\boldsymbol{\epsilon}^j \qquad j=1,\ldots,\tilde{K}
\end{align}
and $\boldsymbol{\epsilon}^j \sim \mathcal{N}(0,\boldsymbol{\Sigma}_{\epsilon})$ are Gaussian i.i.d. 
	\section{Correction Estimation} \label{sec:Correction_term}
The experimental framework detailed in the last section resulted in a set of $\tilde{K}$ noisy point estimates for the correction term ${\tilde{\textbf{F}}}(\mathcal{Y}_m)=\left\lbrace \tilde{\textbf{F}}(\textbf{y}) \vert \textbf{y} \in \mathcal{Y}_m \right\rbrace $ which form our training set. Our interest lies in estimating $\textbf{F}(\cdot)$ over some domain $\mathcal{D}\subseteq\mathbb{R}^d$, however even in the noiseless setting and in the limit where the sampling interval approaches zero, we generally cannot achieve a dense cover over $\mathcal{D}$ with a finite number of trajectories $\textbf{y}^{(k)}(t)$. Thus, some structure or prior information must be assumed for the correction term, such as degree of smoothness or adherence to a specific functional form, to allow for its estimation from the collected data.

In this section we take a Bayesian approach, setting a Gaussian Process (GP) formulation for the problem \cite{kennedy2000predicting}, allowing to express prior knowledge over the correction term $\textbf{F}(\cdot)$ and enabling inference from the finite number of collected noisy samples to the underlying values over the entire domain $\mathcal{D}$. The estimated correction term may subsequently be used to make evolution predictions for arbitrary initial conditions.

\subsection{Gaussian Processes} \label{sec:Gaussian_Processes}
To correct the ODE model we assume a probabilistic setting in which  $\textbf{F}(\textbf{y})$ is a vector-valued GP ${\textbf{F}(\textbf{y})\sim\mathcal{GP}(\textbf{m}(\textbf{y}),\textbf{k}(\textbf{y},\textbf{y}'))}$ defined over some bounded region $\mathcal{D}\subseteq \mathbb{R}^d$ with $\textbf{m}(\cdot):\mathbb{R}^d\rightarrow\mathbb{R}^d$ the mean function and $\textbf{k}(\cdot,\cdot): \mathbb{R}^d\times\mathbb{R}^d\rightarrow \mathbb{R}^{d\times d}$ the covariance function \cite{rasmussen2006gaussian}. Every finite collection of sample points $\left\lbrace \textbf{F}(\textbf{y}^1),\textbf{F}(\textbf{y}^2),\ldots \right\rbrace $ is then distributed as multivariate Gaussian. The mean vector is retrieved by stacking $\textbf{m}(\textbf{y}^1),\textbf{m}(\textbf{y}^2),\ldots$ and the second order statistics are given according to $\mathbb{E}[[\textbf{F}(\textbf{y}^i)]_m[\textbf{F}(\textbf{y}^j)]_n]=[\textbf{k}(\textbf{y}^i,\textbf{y}^j)]_{m,n}$ \cite{alvarez2012kernels}. For the sequel we make the simplifying assumptions $\textbf{m}(\textbf{y})\equiv0$ and $\textbf{k}(\textbf{y},\textbf{y}')=k(\textbf{y},\textbf{y}')\textbf{I}_{d}$, i.e. the vector components are zero mean, independent and share a common scalar kernel function, as in the usual scalar-valued GP setting. Our techniques and methods can be generalized to the biased and correlated-components setting, but we restrict our model here for brevity.

Let $\tilde{\textbf{F}}(\mathcal{A})$ be a set of noisy measurements collected at some set of sampling points $\mathcal{A}$: $\tilde{\textbf{F}}(\mathcal{A})=\left\lbrace \textbf{F}(\textbf{y})+\boldsymbol\epsilon\vert \textbf{y}\in \mathcal{A} \right\rbrace $ where $\boldsymbol\epsilon\sim \mathcal{N}(0,\boldsymbol\Sigma_\epsilon)$ is i.i.d. additive noise. We are interested in predicting the value of the process in unobserved locations. The posterior for $\textbf{F}(\mathcal B)=\left\lbrace \textbf{F}(\textbf{y})\vert \textbf{y}\in \mathcal{B}\right\rbrace$ where $\mathcal{B}$ is some arbitrary set of sampling points is given according to $\textbf{F}(\mathcal B)\vert \tilde{\textbf{F}}(\mathcal{A})\sim\mathcal{N}(\boldsymbol\mu_{\mathcal B\vert \mathcal A},\boldsymbol\Sigma_{\mathcal B\vert \mathcal A})$ with \cite{alvarez2012kernels}:
\begin{align}
\boldsymbol\mu_{\mathcal B\vert \mathcal A}&{=}\textbf{k}(\mathcal{B},\mathcal{A})[\textbf{k}(\mathcal{A},\mathcal{A})+\boldsymbol\Sigma]^{-1}\tilde{\textbf{F}}(\mathcal{A}) \label{eq:gaussian_inference}\\
\boldsymbol\Sigma_{\mathcal B\vert \mathcal A}&{=}\textbf{k}(\mathcal{B},\mathcal{B}){-}\textbf{k}(\mathcal{B},\mathcal{A})[\textbf{k}(\mathcal{A},\mathcal{A}){+}\boldsymbol\Sigma]^{-1}\textbf{k}(\mathcal{A},\mathcal{B})
\end{align}
and $\textbf{k}(\mathcal{S}_1,\mathcal{S}_2)\in\mathbb{R}^{\vert \mathcal{S}_1\vert d\times \vert \mathcal{S}_2\vert d}$ has block structure with elements $[\textbf{k}(\textbf{y}^i,\textbf{y}^j)]_{mn}$ for all $\textbf{y}^i\in \mathcal{S}_1, \textbf{y}^j\in \mathcal{S}_2$ and $m,n=1,\ldots,d$ and $\boldsymbol\Sigma=\boldsymbol\Sigma_\epsilon\otimes \textbf{I}_{\vert \mathcal{A}\vert}$.

The GP formalism facilitates expression of prior knowledge over unknown functions $\textbf{F}(\cdot)$, as determined by the choice of kernel, capturing notions of similarity between values at different positions. Popular choices for the kernel function include the Gaussian RBF $k(\textbf{y},\textbf{y}')=\exp({-}\frac{1}{2\sigma_k^2}\|\textbf{y}{-}\textbf{y}'\|^2)$ with $\sigma_k^2$ the kernel bandwidth and the polynomial kernel $k(\textbf{y},\textbf{y}')=(1+\langle\textbf{y},\textbf{y}'\rangle)^m$ with $m\in\mathbb{N}^{+}$ the order. The Gaussian RBF kernel is of particular interest as it is universal in the sense that with a large enough training set, estimation according to \eqref{eq:gaussian_inference} can approximate any continuous bounded function on a compact domain\cite{micchelli2006universal}. With the GP model set, the value of $\textbf{F}(\textbf{y})$ at any $\textbf{y}\in\mathcal{D}$ may be estimated according to \eqref{eq:gaussian_inference} based on the noisy measurements ${\tilde{\textbf{F}}}(\mathcal{Y}_m)$.

\subsection{Feature Space Representation} \label{sec:feature_space_rep} 
With the assumptions of the last subsection, the $d$-dimensional vector-valued GP $\textbf{F}(\textbf{y})$ is comprised of $d$ independent GPs ${F_i(\textbf{y})\sim\mathcal{GP}(0,k(\textbf{y},\textbf{y}'))},\, i=1,\ldots,d$. We follow \cite{williams1998prediction,rasmussen2006gaussian,de2013extension} and review the correspondence between these GPs and equivalent linear regression models in the feature space.

Indeed, Mercer's theorem guarantees the existence of a sequence of eigenfunctions $\left\lbrace \phi_j(\textbf{y})\right\rbrace, j=1,2,\ldots $ such that $k(\textbf{y},\textbf{y}')=\sum\nolimits_j\phi_j(\textbf{y})\phi_j(\textbf{y}')=\langle\boldsymbol{\phi}(\textbf{y}),\boldsymbol{\phi}(\textbf{y}') \rangle$ where $\boldsymbol{\phi}(\textbf{y})=[\phi_1(\textbf{y}),\phi_2(\textbf{y}),\ldots]^\top$ is the feature transformation from the input space to the feature space and $\left\langle\cdot ; \cdot \right\rangle$ is an inner product.

Let $\theta_{ij}\sim \mathcal{N}(0,1),\, i=1,\ldots,d,\, j=1,2,\ldots$ be a sequence of i.i.d. standard Gaussian variables. For notational convenience we define $\boldsymbol{\theta}_i\equiv[\theta_{i1},\theta_{i2},\ldots ]^{\top},\, i=1,\ldots,d$ and $\boldsymbol{\Theta}=[\boldsymbol{\theta}_1,\ldots,\boldsymbol{\theta}_d]^{\top}$.
We will see that the following identity holds in distribution:
\begin{align} \label{eq:regression_interpretation}
F_i(\textbf{y})=\sum\nolimits_j\theta_{ij}\phi_j(\textbf{y})   \equiv\langle\boldsymbol{\theta}_i,\boldsymbol{\phi}(\textbf{y})\rangle \qquad i=1,\ldots,d
\end{align}
i.e. the GP inference of section \ref{sec:Gaussian_Processes} is equivalent to a Bayesian linear regression model in the feature space. 

To see that \eqref{eq:regression_interpretation} holds notice that both sides of the equality are zeros mean GPs over $\textbf{y}$. The covariance function of the left hand term is $k(\textbf{y},\textbf{y}')$ by definition. The covariance function of the right hand term is 
\begin{align}
\mathbb{E}\left[\sum\nolimits_j\theta_{ij}\phi_j(\textbf{y}) \sum\nolimits_{j'}\theta_{ij'}\phi_{j'}(\textbf{y}') \right]=\sum\nolimits_{jj'}\mathbb{E}\left[\theta_{ij}\theta_{ij'} \right]\phi_j(\textbf{y})\phi_{j'}(\textbf{y}')
=\sum\nolimits_{j}\phi_j(\textbf{y})\phi_{j}(\textbf{y}')=k(\textbf{y},\textbf{y}')
\end{align}

Given noisy data $\tilde{\textbf{F}}(\mathcal{Y}_m)=\textbf{F}(\mathcal{Y}_m)+\boldsymbol\epsilon$ with $\boldsymbol\epsilon \sim \mathcal{N}(0,\boldsymbol\Sigma_{\epsilon})$ i.i.d. noise, inference in the GP can be equivalently performed by estimating 
the regression coefficients $\boldsymbol{\Theta}$ and making predictions for $\textbf{F}(\textbf{y})$ as per \eqref{eq:regression_interpretation}.
	\section{Experimental Design} \label{sec:experimental_design}
In Section \ref{sec:Correction_term} we reviewed inference in a GP setting, and suggested applying this formulation for estimating the correction term $\textbf{F}(\cdot)$ based on the set of noisy measurements $\tilde{\textbf{F}}(\mathcal{Y}_m)$. The observation set $\mathcal{Y}_m$, as determined by the initial conditions set $\mathcal{Y}_0$ was assumed given, and not under our control.

In this section we study efficient experimental design in the misspecified context. That is, our goal would be to select an informative set of experiments, parametrized through the initial conditions $\mathcal{Y}_0$, such as to facilitate rapid learning of the correction term $\textbf{F}(\cdot)$ under a limited experimental budget constraint. We quantify the expected utility associated with choosing sets of initial conditions and suggest an efficient near-optimal (up to a constant factor) algorithm for choosing the best such experimental setup.

\subsection{D-Bayes Optimality} \label{sec:d_bayes} 
Estimation of $\textbf{F}(\textbf{y})$ for generic $\textbf{y}\in \mathcal{D}$ is made possible through the collection of noisy samples $\tilde{\textbf{F}}(\mathcal{Y}_m)$ and application of the methodology of Section \ref{sec:Correction_term}. The quality of inference strongly depends on the sampling set $\mathcal{Y}_m$. For example, if the set $\mathcal{Y}_m$ is highly localized in some region in $\mathcal{D}$ it is reasonable to expect that inference of $\textbf{F}(\cdot)$ becomes more accurate there on the expense of farther locals in $\mathcal{D}$. We are generally interested in estimating $\textbf{F}(\cdot)$ over the whole of $\mathcal{D}$ and so we are interested in developing a mechanism that allows this.

Invoking the feature space representation of Section \ref{sec:feature_space_rep} we see that performing inference in the GP based on a ground set of noisy measurements $\tilde{\textbf{F}}(\mathcal{Y}_m)$ may be viewed as first estimating $\boldsymbol \Theta$ and then applying \eqref{eq:regression_interpretation} to retrieve estimates for the rest of $\mathcal{D}$. From this viewpoint, the estimation error in $\textbf{F}(\textbf{y})$ originates from the error in $\boldsymbol \Theta$ and so our goal is to decrease these as much as possible by maximizing the quality of inference from $\tilde{\textbf{F}}(\mathcal{Y}_m)$ to $\boldsymbol \Theta$. Various statistical criteria have been developed for quantifying the quality of inference between observations and underlying random variables \cite{chaloner1995bayesian,box2011bayesian,fedorov1972theory,pukelsheim1993optimal}. Here we follow D-Bayes optimality~\cite{bernardo1979expected}.

In this framework, the uncertainty associated with $\boldsymbol {\Theta}$ is quantified through the Shannon entropy $H(\cdot)$. Before the experiment we have initial uncertainty $H(\boldsymbol {\Theta})$ which is revised to $H(\boldsymbol {\Theta}|\tilde{\textbf{F}}(\mathcal{Y}_m))$ following data collection. A D-Bayes optimal design minimizes the posterior uncertainty $H(\boldsymbol {\Theta}|\tilde{\textbf{F}}(\mathcal{Y}_m))$, or equivalently maximizes the mutual information:
\begin{align}
I(\boldsymbol {\Theta};\tilde{\textbf{F}}(\mathcal{Y}_m))\equiv H(\boldsymbol {\Theta})-H(\boldsymbol {\Theta}|\tilde{\textbf{F}}(\mathcal{Y}_m))
\end{align}
In our setting we select a set of initial conditions $\mathcal{Y}_0$ and observe the corresponding outputs. This chain of dependencies is made explicit as $\mathcal{Y}_0\rightarrow \mathcal{Y}_m(\mathcal{Y}_0)\rightarrow \tilde{\textbf{F}}(\mathcal{Y}_m(\mathcal{Y}_0))$. The quality of inference, viewed as a function of the initial conditional $\mathcal{Y}_0$ is given by
\begin{align} \label{eq:accurate_cost_function}
G(\mathcal{Y}_0)\equiv I(\boldsymbol {\Theta};\tilde{\textbf{F}}(\mathcal{Y}_m(\mathcal{Y}_0)))
\end{align}
and an optimal experimental design under the budget constraint $\vert\mathcal{Y}_0\vert\leq K, \,\mathcal{Y}_0 \subseteq \mathcal{Y}$ is
\begin{align} \label{eq:opt_prob_exp_des}
\mathcal{Y}_0^\star= \underset{\mathcal{Y}_0: |\mathcal{Y}_0|\leq K, \mathcal{Y}_0\subseteq \mathcal{Y}}{\text{argmax}} G(\mathcal{Y}_0)
\end{align}

\subsection{Output Trajectory Proxy} The design problem \eqref{eq:opt_prob_exp_des} entails choosing a set $\mathcal{Y}_0$ of $K$ initial conditions, and observing $\tilde{K}$ noisy measurements $\tilde{\textbf{F}}(\mathcal{Y}_m(\mathcal{Y}_0))$, which are utilized for estimating $\textbf{F}(\cdot)$ over $\mathcal{D}$.

As we are concerned with misspecified systems such that the complete system model \eqref{eq:mis_specified_model} is unknown, we are unable to predict system trajectories based on initial conditions at time zero. In particular, we do not have a-priori access to the mapping between the sets $\mathcal{Y}_0$ and $\mathcal{Y}_m$, such that evaluation of the cost function \eqref{eq:accurate_cost_function} and thus solution of the design problem \eqref{eq:opt_prob_exp_des} are not possible. However, at this point our assumption that the system is only slightly misspecified in short time spans, i.e. that the correction term $\textbf{F}(\cdot)$ introduces a small effect on the trajectory, turns out to be useful in retrieving approximate solutions.
 
For any given set of initial conditions $\mathcal{Y}_0$ we invoke the approximate system model \eqref{eq:approx_sys_model} to obtain a proxy $\mathcal{Y}_g$ for the true set of future states $\mathcal{Y}_m$. Let $\textbf{y}^{(k)}(0)\in\mathcal{Y}_0$ be the initial conditions seeding the $\textit{k}^{\rm{th}}$ experiment, and designate the approximate ensuing trajectory $\textbf{y}_G^{(k)}(t)$. Collect the approximate trajectories in  $\mathcal{Y}_g\equiv\left\lbrace 
\textbf{y} \vert \exists k,i \enspace \text{s.t.}\enspace \textbf{y}=\textbf{y}_G^{(k)}(t_i)\right\rbrace$, and note that the set $\mathcal{Y}_g$ may be evaluated in advance given $\mathcal{Y}_0$. For example, for a linear misspecified system $\frac{d}{dt}\textbf{y}_G(t)=\textbf{A}\textbf{y}_G(t)$ for some fixed $\textbf{A}\in\mathbb{R}^{d\times d}$, the trajectories comprising $\mathcal{Y}_g$ may be determined according to $\textbf{y}_G^{(k)}(t_i)=e^{\textbf{A}t_i}\textbf{y}_G^{(k)}(0)$ where $e^{(\cdot)}$ is the matrix exponential according to the usual definition.

In what follows we propose a proxy for the cost function \eqref{eq:accurate_cost_function} where $\mathcal{Y}_g$ is used in lieu of the unknown $\mathcal{Y}_m$, and derive approximation bounds for the discrepancy between the two. We show that that these bounds scale with the deviation between the actual and approximate system outputs $\textbf{y}(\cdot)$ and $\textbf{y}_G(\cdot)$. Specifically, we have:
\begin{theorem} \label{thm:1}
Let $\tilde{G}(\mathcal{Y}_0)\equiv I(\boldsymbol {\Theta};\tilde{\textbf{F}}(\mathcal{Y}_g(\mathcal{Y}_0)))$, with $\boldsymbol\epsilon \sim \mathcal{N}(0,\boldsymbol\Sigma_{\epsilon})$, and assume that the maximum covariance discrepancy between the true and approximate models is bounded such that 
$$\forall k_1,k_2,i_1,i_2: \left\lvert k(\textbf{y}^{(k_1)}(t_{i_1}),\textbf{y}^{(k_2)}(t_{i_2}))-k(\textbf{y}_G^{(k_1)}(t_{i_1}),\textbf{y}_G^{(k_2)}(t_{i_2})) \right\rvert\leq \delta\,.$$ We have
\begin{align} \label{eq:thm_bound}
\left\lvert \tilde{G}(\mathcal{Y}_0)-G(\mathcal{Y}_0) \right\rvert \leq- d\tilde{K}\log\left(1{-}\frac{\delta {(d\tilde{K})}^{\frac32}}{\sigma_{\min}({\boldsymbol \Sigma_\epsilon})}\right)
\end{align}	
with $\sigma_{\min}(\cdot)$ the minimal singular value.
\end{theorem}	
\begin{proof}
Using the definition of mutual information\footnote{$I(x;y){=}H(x){-}H(x\vert y){=}H(y){-}H(y\vert x)$} we have
\begin{align}
G(\mathcal{Y}_0)=I(\boldsymbol {\Theta};\tilde{\textbf{F}}(\mathcal{Y}_m(\mathcal{Y}_0)))&=H(\tilde{\textbf{F}}(\mathcal{Y}_m(\mathcal{Y}_0)))-H(\tilde{\textbf{F}}(\mathcal{Y}_m(\mathcal{Y}_0))\vert \boldsymbol {\Theta}) \\
\tilde{G}(\mathcal{Y}_0)=I(\boldsymbol {\Theta};\tilde{\textbf{F}}(\mathcal{Y}_g(\mathcal{Y}_0)))&=H(\tilde{\textbf{F}}(\mathcal{Y}_g(\mathcal{Y}_0)))-H(\tilde{\textbf{F}}(\mathcal{Y}_g(\mathcal{Y}_0))\vert \boldsymbol {\Theta})
\end{align}
Conditioned on $\boldsymbol{\Theta}$ the remaining uncertainty in the measurements is just the random noise and we have $H(\tilde{\textbf{F}}(\mathcal{Y}_m(\mathcal{Y}_0))\vert \boldsymbol {\Theta})=H(\tilde{\textbf{F}}(\mathcal{Y}_g(\mathcal{Y}_0))\vert \boldsymbol {\Theta})=H( \boldsymbol{\epsilon} )$ such that:
\begin{align}
G(\mathcal{Y}_0)-\tilde{G}(\mathcal{Y}_0)=H(\tilde{\textbf{F}}(\mathcal{Y}_m(\mathcal{Y}_0)))-H(\tilde{\textbf{F}}(\mathcal{Y}_g(\mathcal{Y}_0)))
\end{align}
Notice that both $\tilde{\textbf{F}}(\mathcal{Y}_m(\mathcal{Y}_0))$ and $\tilde{\textbf{F}}(\mathcal{Y}_g(\mathcal{Y}_0))$ are collections of $\tilde{K}$ Gaussian random variables as noisy samples from the GP. Now apply the generic formula for the entropy of a Gaussian random vector\footnote{$\textbf{x}\in \mathbb{R}^k ,\textbf{x}\sim \mathcal{N}(\boldsymbol{\mu},\boldsymbol \Sigma) \Rightarrow H(\textbf{x})=\log((\pi e)^{k} \text{det}\boldsymbol \Sigma)$}:
\begin{align} 
H(\tilde{\textbf{F}}(\mathcal{Y}_m(\mathcal{Y}_0)))&=\log((\pi e)^{\tilde{K}} \text{det}{\boldsymbol \Sigma_m})\notag\\
H(\tilde{\textbf{F}}(\mathcal{Y}_g(\mathcal{Y}_0)))&=\log((\pi e)^{\tilde{K}} \text{det}{\boldsymbol \Sigma_g}) 
\end{align}
where $\tilde K\equiv d\tilde{K}$ and:
\begin{align} \label{eq:cov_mat_def}
\boldsymbol \Sigma_m&=\textbf{k}(\mathcal{Y}_m,\mathcal{Y}_m)+\boldsymbol{\Sigma}_{\epsilon}\otimes \textbf{I}_{\tilde{K}}\notag\\
\boldsymbol \Sigma_g&=\textbf{k}(\mathcal{Y}_g,\mathcal{Y}_g)+\boldsymbol{\Sigma}_{\epsilon}\otimes \textbf{I}_{\tilde{K}}
\end{align}
So, \begin{align} \label{eq:simplify_diff}
\tilde{G}(\mathcal{Y}_0)-G(\mathcal{Y}_0)=\log(\text{det}(\boldsymbol \Sigma_g))-\log(\text{det}(\boldsymbol \Sigma_m))
\end{align}
Now define $\textbf{X}\equiv\frac{1}{\delta}(\boldsymbol \Sigma_g-\boldsymbol \Sigma_m) \, \leftrightarrow \, \boldsymbol \Sigma_g=\boldsymbol \Sigma_m+\delta \textbf{X}$ with $\textbf{X} \in \mathbb{R}^{\tilde{K}\times \tilde{K}}$ satisfying $\forall i,j\; \vert X_{ij}\vert\leq 1$ according to our assumption of bounded covariance differences. 

$\boldsymbol \Sigma_{g},\boldsymbol \Sigma_{m} $ are both positive-definite and invertible such that we can write: 
\begin{align} \label{eq:cov_matrix_trunc}
\text{det}(\boldsymbol \Sigma_g)=\text{det}(\boldsymbol \Sigma_m+\delta\textbf{X})=\text{det}(\boldsymbol \Sigma_m)\text{det}(\textbf{I}+\delta \boldsymbol \Sigma_m^{-1} \textbf{X} )
\end{align} 
Substituting \eqref{eq:cov_matrix_trunc} in \eqref{eq:simplify_diff} we have:
\begin{align}
\tilde{G}(\mathcal{Y}_0)-G(\mathcal{Y}_0)=\log(\text{det}(\textbf{I}+\delta \boldsymbol \Sigma_m^{-1} \textbf{X} )){=}\log(\text{det}(\tilde{\textbf{X}} ))
\end{align}
with $\tilde{\textbf{X}}\equiv \textbf{I}+\delta \boldsymbol \Sigma_m^{-1} \textbf{X}$. We turn next to bounding $\log(\text{det}(\tilde{\textbf{X}} ))$. First notice:
\begin{eqnarray*} 
\left\lvert[\delta\boldsymbol \Sigma_m^{-1} \textbf{X}]_{ij}\right\rvert & {=} & \delta \left\lvert \sum\limits_{r}\Sigma_{m,ir}^{-1}X_{rj}\right\rvert \leq \delta \sum\limits_{r} \left\lvert \Sigma_{m,ir}^{-1}X_{rj}\right\rvert \\ 
& {\leq} & \delta \sum\limits_{r} \left\lvert{\Sigma}^{-1}_{m,ir}\right\rvert{\leq} \delta \|{\boldsymbol \Sigma}_m^{-1} \|_{\infty}\notag\\ 
&{\leq}& \delta\sqrt{\tilde{K}} \|{\boldsymbol \Sigma}_m^{-1} \|_{2}
=\delta\sqrt{\tilde{K}} \sigma_{\text{max}}({\boldsymbol \Sigma}_m^{-1})\\ & =& \frac{\delta\sqrt{\tilde{K}}}{\sigma_{\text{min}}({\boldsymbol \Sigma}_m)}\leq \frac{\delta\sqrt{\tilde{K}}}{\sigma_{\text{min}}({\boldsymbol \Sigma_\epsilon})}
\end{eqnarray*}
where we used the matrix norm inequality $\| \textbf{A}\|_{\infty}{\leq} \sqrt{\tilde{K}} \|\textbf{A} \|_{2}$ for $\textbf{A}\in\mathbb{R}^{\tilde{K} {\times} \tilde{K}}$ and $\sigma_{\text{max}}(\cdot)$.

Thus we have that $\tilde{\textbf{X}}$ has diagonal elements centered around 1, i.e. for all $i$ $$\left\lvert\tilde{X}_{ii}-1\right\rvert\leq \frac{\delta\sqrt{\tilde{K}}}{\sigma_{\text{min}}({\boldsymbol \Sigma_\epsilon})}$$ and the row-sums over non-diagonal entries satisfy for all $i$ $$\sum\limits_{r\neq i} \left\lvert\tilde{X}_{ir} \right\rvert \leq\frac{\delta\sqrt{\tilde{K}}(\tilde{K}-1)}{\sigma_{\text{min}}({\boldsymbol \Sigma_\epsilon})}$$.
Designating the eigenvalues of $\tilde{\textbf{X}}$ as $\left\lbrace \lambda_i\right\rbrace$ and applying the
Gershgorin circle theorem, we have:
\begin{align}
1-\frac{\delta {\tilde{K}}^{\frac32}}{\sigma_{\text{min}}({\boldsymbol \Sigma_\epsilon})} \leq \vert\lambda_i\vert \leq 1+\frac{\delta {\tilde{K}}^{\frac32}}{\sigma_{\text{min}}({\boldsymbol \Sigma_\epsilon})}\,.
\end{align}
Using 	$\log(\text{det}(\tilde{\textbf{X}} ))=\sum\limits_{i}\log(\left\lvert\lambda_i\right\rvert)$, this implies that
\begin{align}
\tilde{K}\log\left(1-\frac{\delta {\tilde{K}}^{\frac32}}{\sigma_{\text{min}}({\boldsymbol \Sigma_\epsilon})}\right)  \leq\log(\text{det}(\tilde{\textbf{X}} ))\leq \tilde{K}\log\left(1+\frac{\delta {\tilde{K}}^{\frac32}}{\sigma_{\text{min}}({\boldsymbol \Sigma_\epsilon})}\right)
\end{align}
where the left hand side is to be interpreted as minus infinity when the argument of the log function is negative. Finally, using $\tilde{G}(\mathcal{Y}_0)-G(\mathcal{Y}_0)=\log(\text{det}(\tilde{\textbf{X}} ))$ and $\log(1{+}x){\leq}{-}\log(1{-}x)$ we have:
\begin{align}
\vert \tilde{G}(\mathcal{Y}_0)-G(\mathcal{Y}_0) \vert \leq-\tilde{K}\log\left(1{-}\frac{\delta {\tilde{K}}^{\frac32}}{\sigma_{\text{min}}({\boldsymbol \Sigma_\epsilon})}\right)
\end{align}
\end{proof}


Notice that the bound of Theorem \ref{thm:1} becomes looser as the noise decreases. That is, notice that the value of $G(\mathcal{Y}_0)$ increases in this case in about the same proportion so the relative error remains similar. As an illustration, consider the case $\boldsymbol{\Sigma}_\epsilon=\sigma_\epsilon^2\textbf{I}$. Using notation used in the proof of Theorem~\ref{thm:1} and $\boldsymbol\Sigma\equiv\boldsymbol{\Sigma}_{\epsilon}\otimes \textbf{I}_{\tilde{K}}$ we have:
\begin{eqnarray*}
G(\mathcal{Y}_0)& {=} & H(\tilde{\textbf{F}}(\mathcal{Y}_m(\mathcal{Y}_0))){-}H(\tilde{\textbf{F}}(\mathcal{Y}_m(\mathcal{Y}_0))\vert \boldsymbol {\Theta})\notag\\
&{=}& \log(\text{det}(\textbf{k}(\mathcal{Y}_m,\mathcal{Y}_m){+}\boldsymbol{\Sigma})){-}\log(\text{det}(\boldsymbol{\Sigma}))\\
&=&\log(\text{det}( \textbf{I}{+} \boldsymbol{\Sigma}^{-1}\textbf{k}(\mathcal{Y}_m,\mathcal{Y}_m)))
\end{eqnarray*}
Now observe $$\lambda_i(\textbf{I}{+} \boldsymbol{\Sigma}^{-1}\textbf{k}(\mathcal{Y}_m,\mathcal{Y}_m))= 1{+}\lambda_i(\boldsymbol{\Sigma}^{-1}\textbf{k}(\mathcal{Y}_m,\mathcal{Y}_m)){\geq}1{+}\frac{\sigma_{\min}(\textbf{k}(\mathcal{Y}_m,\mathcal{Y}_m))}{\sigma_\epsilon^2}
$$
so  $$\log(\text{det}( \textbf{I}{+} \boldsymbol{\Sigma}^{-1}\textbf{k}(\mathcal{Y}_m,\mathcal{Y}_m) ))\geq\tilde{K}\log\left(1{+}\frac{\sigma_{\min}(\textbf{k}(\mathcal{Y}_m,\mathcal{Y}_m))}{\sigma_\epsilon^2}\right)$$
and we have 
\begin{align*}
G(\mathcal{Y}_0)\geq  d\tilde{K}\log\left(1{+}\frac{\sigma_{\min}(\textbf{k}(\mathcal{Y}_m(\mathcal{Y}_0),\mathcal{Y}_m(\mathcal{Y}_0)))}{\sigma_\epsilon^2}\right)\,.
\end{align*}

\begin{corollary} \label{cor:gaussian}
Let $k(\textbf{y},\textbf{y}')=k(\|\textbf{y}{-}\textbf{y}' \|)$ be a shift-invariant kernel with $k(\cdot)$ Lipschitz continuous with constant $L$ over $\mathcal{D}'\equiv\left\lbrace \textbf{y}^1{-}\textbf{y}^2\vert \textbf{y}^1,\textbf{y}^2\in\mathcal{D} \right\rbrace $, and assume $\forall k,i: \| \textbf{y}^{(k)}(t_{i})-\textbf{y}_G^{(k)}(t_{i})) \|{\leq} \Delta$.
We have 
\begin{align}
\left\lvert \tilde{G}(\mathcal{Y}_0)-G(\mathcal{Y}_0) \right\rvert \leq- d\tilde{K}\log\left(1{-}\frac{2L\Delta {(d\tilde{K})}^{\frac32}}{\sigma_{\min}({\boldsymbol \Sigma_\epsilon})}\right)
\end{align}
\end{corollary}
\begin{proof}
For any $k_1,k_2,i_1,i_2$ we have 
\begin{eqnarray*}
& & \left\lvert k(\textbf{y}^{(k_1)}(t_{i_1}),\textbf{y}^{(k_2)}(t_{i_2})){-}k(\textbf{y}_G^{(k_1)}(t_{i_1}),\textbf{y}_G^{(k_2)}(t_{i_2})) \right\rvert \\
&{=}&
\left\lvert k(\|\textbf{y}^{(k_1)}(t_{i_1}){-}\textbf{y}^{(k_2)}(t_{i_2})\|){-}k(\|\textbf{y}_G^{(k_1)}(t_{i_1}){-}\textbf{y}_G^{(k_2)}(t_{i_2})\|) \right\rvert\notag\\
&\leq&
L\left\lvert\|\textbf{y}^{(k_1)}(t_{i_1}){-}\textbf{y}^{(k_2)}(t_{i_2})\|{-}\|\textbf{y}_G^{(k_1)}(t_{i_1}){-}\textbf{y}_G^{(k_2)}(t_{i_2})\|\right\rvert\\
&\leq&
L\|( \textbf{y}^{(k_1)}(t_{i_1}){-}\textbf{y}_G^{(k_1)}(t_{i_1})){-}(\textbf{y}^{(k_2)}(t_{i_2}) {-}\textbf{y}_G^{(k_2)}(t_{i_2}))\|\notag\\
&\leq &L \left(\| \textbf{y}^{(k_1)}(t_{i_1}){-}\textbf{y}_G^{(k_1)}(t_{i_1})\|{+}\|\textbf{y}^{(k_2)}(t_{i_2}) {-}\textbf{y}_G^{(k_2)}(t_{i_2}) \| \right) \\
& \leq&
2L\Delta
\end{eqnarray*}
and the result follows by substitution in \eqref{eq:thm_bound}.
\end{proof}
	
\begin{corollary} \label{cor:poly}
Let $k(\textbf{y},\textbf{y}')=(1+\left\langle \textbf{y},\textbf{y}'\right\rangle)^m$ be the polynomial kernel, $B\equiv\sup_{\textbf{y}\in \mathcal{D}}\|\textbf{y}\|$ and assume $\forall k,i: \| \textbf{y}^{(k)}(t_{i})-\textbf{y}_G^{(k)}(t_{i})) \|{\leq} \Delta$, then \begin{align}
\left\lvert \tilde{G}(\mathcal{Y}_0)-G(\mathcal{Y}_0) \right\rvert \leq- d\tilde{K}\log\left(1{-}\frac{m\Delta(2B+\Delta)(1+B^2)^{m-1} {(d\tilde{K})}^{\frac32}}{\sigma_{\min}({\boldsymbol \Sigma_\epsilon})}\right)
\end{align}
\end{corollary}
\begin{proof}
Consider the following chain of inequalities
\begin{eqnarray*}
	& &\left\lvert k(\textbf{y}^{(k_1)}(t_{i_1}),\textbf{y}^{(k_2)}(t_{i_2})){-}k(\textbf{y}_G^{(k_1)}(t_{i_1}),\textbf{y}_G^{(k_2)}(t_{i_2})) \right\rvert\notag \\
	&=&
	\left\lvert
	\left(1+\left\langle \textbf{y}^{(k_1)}(t_{i_1}),\textbf{y}^{(k_2)}(t_{i_2})\right\rangle \right)^m-\left(1+\left\langle \textbf{y}_G^{(k_1)}(t_{i_1}),\textbf{y}_G^{(k_2)}(t_{i_2})\right\rangle \right)^m \right\rvert\notag\\
	&\overset{(a)}\leq & m(1+B^2)^{m-1}\left\lvert \left\langle \textbf{y}^{(k_1)}(t_{i_1}),\textbf{y}^{(k_2)}(t_{i_2})\right\rangle - \left\langle \textbf{y}_G^{(k_1)}(t_{i_1}),\textbf{y}_G^{(k_2)}(t_{i_2})\right\rangle \right\rvert\notag\\ 
	&=&m(1{+}B^2)^{m{-}1}\vert
	\left\langle \textbf{y}_G^{(k_1)}(t_{i_1}){-}\textbf{y}^{(k_1)}(t_{i_1}),\textbf{y}^{(k_2)}(t_{i_2})\right\rangle\\
	& & {+}
	\quad\left\langle \textbf{y}^{(k_1)}(t_{i_1}),\textbf{y}_G^{(k_2)}(t_{i_2}){-}\textbf{y}^{(k_2)}(t_{i_2})\right\rangle\notag\\
	& & {+}
	\quad\left\langle \textbf{y}_G^{(k_1)}(t_{i_1}){-}\textbf{y}^{(k_1)}(t_{i_1}),\textbf{y}_G^{(k_2)}(t_{i_2}){-}\textbf{y}^{(k_2)}(t_{i_2}))\right\rangle
	\vert\\
	&{\leq}&
	m(1{+}B^2)^{m{-}1}(\Delta B {+} B\Delta{+}\Delta^2)\notag\\
	&=&
	m\Delta(2B{+}\Delta)(1{+}B^2)^{m-1}
\end{eqnarray*}
where (a) is due to the Lipschitz constant of the function $f(x)=(1+x)^m$ being smaller than $m(1+\sup_{x\in\mathcal{D}} \vert x\vert)^{m-1}$. The result follows by substitution in \eqref{eq:thm_bound}.
\end{proof}

Theorem \ref{thm:1} and Corollaries \ref{cor:gaussian} and \ref{cor:poly} bound the discrepancy between $G(\mathcal{Y}_0)$ and its proxy $\tilde{G}(\mathcal{Y}_0)$. As the trajectory uncertainty becomes smaller the two become more tightly aligned as quantified by our results in this subsection.

\subsection{Near Optimal Solution}
Based on the results of Theorem \ref{thm:1} and the ensuing corollaries, in lieu of problem \eqref{eq:opt_prob_exp_des} we pose a relaxed proxy that circumvents around the uncertainty associated with the system output. Namely, we are interested in the solution of
\begin{align} \label{eq:opt_prob_exp_des_app}
\tilde{\mathcal{Y}}_0^{\star}= \underset{\mathcal{Y}_0: |\mathcal{Y}_0|\leq K, \mathcal{Y}_0\subseteq \mathcal{Y}}{\text{argmax}} \tilde{G}(\mathcal{Y}_0)
\end{align}

Generic combinatorial optimization problems such as \eqref{eq:opt_prob_exp_des_app} exhibit prohibitive computational complexity, as the solution generally involves enumeration over all possible subset combinations satisfying the constraints, which is exponential in the size of the set $\vert\mathcal{Y}_0\vert$. We prove that $\tilde{G}(\mathcal{Y}_0)$ holds favorable properties, rendering the optimization problem \eqref{eq:opt_prob_exp_des_app} amenable to approximate solution by means of computationally efficient algorithms with provable guarantees. We start with some useful definitions:
\begin{definition}
Let $\mathcal{V}$ be a set and $G: 2^{\mathcal{V}}\rightarrow \mathbb{R}$ a set function.
\begin{enumerate}
\item G  is \textbf{submodular} if it satisfies the property of decreasing marginals: $\; \forall \mathcal{S},\mathcal{T} {\subseteq} \mathcal{V}$ such that $\mathcal{S} {\subseteq} \mathcal{T}$ and $x{\in} \mathcal{V}{\backslash} \mathcal{T}$ it holds that $G(\mathcal{S}{\cup}\left\lbrace x \right\rbrace ){-}G(\mathcal{S}) {\geq} G(\mathcal{T}{\cup}\left\lbrace x \right\rbrace ){-}G(\mathcal{T})$.
\item $G$ is \textbf{monotonic} (increasing) if $\; \forall \mathcal{S}, \mathcal{T} {\subseteq} \mathcal{V}$ s.t. $\mathcal{S}{\subseteq} \mathcal{T}$ we have $G(\mathcal{S}){\leq} G(\mathcal{T})$.
\end{enumerate}
\end{definition}
Our next step is to show that the set function $\tilde{G}(\mathcal{Y}_0)$ is submodular and monotonic (similar to \cite{krause2012near}), a fact that allows us to make use of the rich literature on submodular optimization.
\begin{theorem}
Let $\tilde{G}:2^{\mathcal{Y}}\rightarrow \mathbb{R}$ be the set function defined in Theorem \ref{thm:1}. Then $\tilde{G}$ is monotonic (increasing) and submodular.
\end{theorem}
\begin{proof}
First we prove submodularity. Let $\mathcal{Y}_0 \subset \mathcal{Y}$ and $\textbf{y}\in \mathcal{Y}\backslash \mathcal{Y}_0$, such that the system output proxy for $\textbf{y}$ is given as $\tilde{\textbf{F}}(\mathcal{Y}_g(\textbf{y}))$. 
Expanding the mutual information according to $I(\boldsymbol {\Theta};\tilde{\textbf{F}}(\mathcal{Y}_g(\mathcal{Y}_0)))=H(\tilde{\textbf{F}}(\mathcal{Y}_g(\mathcal{Y}_0))){-}H(\tilde{\textbf{F}}(\mathcal{Y}_g(\mathcal{Y}_0))\vert \boldsymbol {\Theta})$ we have:
\begin{eqnarray} \label{eq:mi_marginal}
\tilde{G}(\mathcal{Y}_0{\cup}\left\lbrace \textbf{y} \right\rbrace ){-}\tilde{G}(\mathcal{Y}_0)& = & 
H(\tilde{\textbf{F}}(\mathcal{Y}_g(\mathcal{Y}_0)){\cup} \tilde{\textbf{F}}(\mathcal{Y}_g(\textbf{y}))  )
{-}
H(\tilde{\textbf{F}}(\mathcal{Y}_g(\mathcal{Y}_0))
{-} \notag\\
& & \quad [H(\tilde{\textbf{F}}(\mathcal{Y}_g(\mathcal{Y}_0)){\cup} \tilde{\textbf{F}}(\mathcal{Y}_g(\textbf{y})\vert \boldsymbol {\Theta}){-}H(\tilde{\textbf{F}}(\mathcal{Y}_g(\mathcal{Y}_0)\vert \boldsymbol {\Theta})]\\
&= & H(\tilde{\textbf{F}}(\mathcal{Y}_g(\textbf{y}))  \vert \tilde{\textbf{F}}(\mathcal{Y}_g(\mathcal{Y}_0))   )-H(\tilde{\textbf{F}}(\mathcal{Y}_g(\textbf{y}))  \vert \boldsymbol {\Theta})\notag
\end{eqnarray}
where we used the conditional independence of the elements of $\tilde{\textbf{F}}(\mathcal{Y}_g(\mathcal{Y}_0))\cup \tilde{\textbf{F}}(\mathcal{Y}_g(\textbf{y})) $ given $ \boldsymbol {\Theta}$, so $H(\tilde{\textbf{F}}(\mathcal{Y}_g(\mathcal{Y}_0)){\cup} \tilde{\textbf{F}}(\mathcal{Y}_g(\textbf{y}))  \vert \boldsymbol {\Theta})=H(\tilde{\textbf{F}}(\mathcal{Y}_g(\mathcal{Y}_0)) \vert \boldsymbol {\Theta})+H( \tilde{\textbf{F}}(\mathcal{Y}_g(\textbf{y})) \vert \boldsymbol {\Theta})$.

Now apply the results of \eqref{eq:mi_marginal} twice, for two specific choices for $\mathcal{Y}_0$, namely $\mathcal{Y}_0\leftarrow \mathcal{Y}_0^1$ and $\mathcal{Y}_0\leftarrow \mathcal{Y}_0^2$ such that $\mathcal{Y}_0^1\subseteq \mathcal{Y}_0^2$:
\begin{eqnarray*}
& &[\tilde{G}(\mathcal{Y}_0^1{\cup}\left\lbrace \textbf{y} \right\rbrace ){-}\tilde{G}(\mathcal{Y}_0^1)]{-}[\tilde{G}(\mathcal{Y}_0^2{\cup}\left\lbrace \textbf{y} \right\rbrace ){-}\tilde{G}(\mathcal{Y}_0^2)]\\
&=& 
H(\tilde{\textbf{F}}(\mathcal{Y}_g(\textbf{y}))\vert  \tilde{\textbf{F}}(\mathcal{Y}_g(\mathcal{Y}_0^1)))-H(\tilde{\textbf{F}}(\mathcal{Y}_g(\textbf{y}))\vert  \tilde{\textbf{F}}(\mathcal{Y}_g(\mathcal{Y}_0^2)))
\end{eqnarray*}
Conditioning on a larger set cannot increase entropy and we have $H(\tilde{\textbf{F}}(\mathcal{Y}_g(\textbf{y}))\vert  \tilde{\textbf{F}}(\mathcal{Y}_g(\mathcal{Y}_0^1))) \geq
H(\tilde{\textbf{F}}(\mathcal{Y}_g(\textbf{y}))\vert  \tilde{\textbf{F}}(\mathcal{Y}_g(\mathcal{Y}_0^2)))$
such that $\tilde{G}(\mathcal{Y}_0^1{\cup}\left\lbrace \textbf{y} \right\rbrace ){-}\tilde{G}(\mathcal{Y}_0^1) \geq \tilde{G}(\mathcal{Y}_0^2\cup\left\lbrace \textbf{y} \right\rbrace ){-}\tilde{G}(\mathcal{Y}_0^2)$ and $\tilde{G}$ is submodular.
	
To prove monotonicity it is enough to show $\tilde{G}(\mathcal{Y}_0{\cup}\left\lbrace \textbf{y} \right\rbrace ){-}\tilde{G}(\mathcal{Y}_0)\geq 0$. This time expand the mutual information according to $I(\boldsymbol {\Theta};\tilde{\textbf{F}}(\mathcal{Y}_g(\mathcal{Y}_0)))=H(\boldsymbol {\Theta})-H(\boldsymbol {\Theta}\vert\tilde{\textbf{F}}(\mathcal{Y}_g(\mathcal{Y}_0)))$:
\begin{align}
\tilde{G}(\mathcal{Y}_0{\cup}\left\lbrace \textbf{y} \right\rbrace ){-}\tilde{G}(\mathcal{Y}_0)=
H(\boldsymbol {\Theta}\vert\tilde{\textbf{F}}(\mathcal{Y}_g(\mathcal{Y}_0)))
{-}
H(\boldsymbol {\Theta}\vert\tilde{\textbf{F}}(\mathcal{Y}_g(\mathcal{Y}_0)){\cup} \tilde{\textbf{F}}(\mathcal{Y}_g(\textbf{y})))\,.
\end{align} 
Conditioning can never increase entropy so $$H(\boldsymbol {\Theta}\vert\tilde{\textbf{F}}(\mathcal{Y}_g(\mathcal{Y}_0))) \geq H(\boldsymbol {\Theta}\vert\tilde{\textbf{F}}(\mathcal{Y}_g(\mathcal{Y}_0)){\cup} \tilde{\textbf{F}}(\mathcal{Y}_g(\textbf{y})))$$ and the result follows.
\end{proof}

\begin{algorithm}[t]
	\caption{Greedy Submodular Maximization}
	\label{al:greedy_max}
	\begin{algorithmic}
		\STATE $\mathcal{S}\leftarrow \emptyset$
		\FOR{$i=1$ to $K$}
		\STATE $x^{\star}=\text{argmax}_{x\in \mathcal{V}\setminus \mathcal{S}}\,G(\mathcal{S}\cup \left\lbrace x \right\rbrace )\quad$ See Equation \eqref{eq:algorithm_prog} 
		\STATE $\mathcal{S} \leftarrow \mathcal{S} \cup \left\lbrace x^{\star}\right\rbrace$
		\ENDFOR
		\STATE Return $\mathcal{S}$
	\end{algorithmic}
\end{algorithm}

The class of submoudlar combinatorial optimization problems has been extensively studied in the past \cite{fujishige2005submodular}. While submodular optimization problems are known to be NP-hard, it is known that the computationally efficient greedy solver delineated in algorithm \ref{al:greedy_max} is guaranteed to achieve a good approximation (up to a constant factor) to the optimal solution \cite{minoux1978accelerated, nemhauser1978analysis}, as stated in the next lemma:
\begin{lemma}[Nemhauser \cite{nemhauser1978analysis}] \label{lem:nemhauser} Let $G$ be a monotonic, submodular set function. Let $\mathcal{S}^{\star}=\underset{\mathcal{S}\subseteq \mathcal{V},\vert \mathcal{S}\vert\leq K}{\text{argmax}}\,G(\mathcal{S})$ be an optimal solution and $\mathcal{S}^{\text{gr}}$ a set retrieved by the greedy maximization algorithm \ref{al:greedy_max}. We have the following guarantee for the performance of the greedy algorithm:
$$G(\mathcal{S}^{\text{gr}})\geq (1-e^{-1})G(\mathcal{S}^{\star})$$
Moreover, no polynomial time algorithm can provide a better approximation guarantee
unless P=NP \cite{feige1998threshold}.
\end{lemma}

To determine the computational complexity of Algorithm \ref{al:greedy_max} notice that we have $K$ iterations, where in each iteration we evaluate $\vert \mathcal{V}\setminus \mathcal{S} \vert \approx \vert \mathcal{V}\vert$ candidate sets of the form $G(\mathcal{S}\cup \left\lbrace x \right\rbrace)$ to determine the one of biggest value. The computational cost is thus $O(K\vert \mathcal{V} \vert C)$ where $C$ is the cost of evaluating candidate sets. For example, for our application $C$ is the cost of a single evaluation of Equation \eqref{eq:algorithm_prog} determining the mutual information of a candidate set.

\paragraph{Proposed Method}
Applied to our setting, the Algorithm \ref{al:greedy_max} performs successive evaluations of the proxy function $\tilde{G}(\cdot)$ for candidate sets $\mathcal{Y}_0^C\equiv\mathcal{Y}_0\cup \left\lbrace \textbf{y} \right\rbrace$ where $\textbf{y} \in \mathcal{Y}\setminus \mathcal{Y}_0$. During the $\textit{k}^{\rm{th}}$ iteration the candidate sets $\mathcal{Y}_0^C$ are of size $k$. We utilize the following identity to facilitate the flow of the algorithm:
\begin{align} \label{eq:algorithm_prog}
\tilde{G}(\mathcal{Y}_0^C)=&H(\tilde{\textbf{F}}(\mathcal{Y}_g(\mathcal{Y}_0^C))){-}H(\tilde{\textbf{F}}(\mathcal{Y}_g(\mathcal{Y}_0^C))\vert \boldsymbol {\Theta})\notag\\
=&\log((\pi e)^{kT} \text{det}\boldsymbol \Sigma_g)-\log((\pi e)^{kT} \text{det}\boldsymbol \Sigma_{g\vert \boldsymbol{\Theta}})=\log(\text{det}\boldsymbol \Sigma_g)-\log(\text{det}\boldsymbol \Sigma_{g\vert \boldsymbol{\Theta}})
\end{align}
In the above, $\boldsymbol\Sigma_g$ and $\boldsymbol \Sigma_{g\vert \boldsymbol{\Theta}}$ are the covariance matrices for the ensemble of $kT$ samples $\tilde{\textbf{F}}(\mathcal{Y}_g(\mathcal{Y}_0^C))$, taken without and with conditioning on the feature space coefficients $\boldsymbol {\Theta}$, respectively. Notice that conditioned on $\boldsymbol{\Theta}$ the measurements covariance matrix $\boldsymbol \Sigma_{g\vert \boldsymbol{\Theta}}$ is block-diagonal with block submatrices being the noise covariance matrix, and the no-conditioning covariance matrix $\boldsymbol\Sigma_g$ can be retrieved by adding the aforementioned noise matrix to the corresponding kernel covariance matrix $k(\tilde{\textbf{F}}(\mathcal{Y}_g(\mathcal{Y}_0^C)),\tilde{\textbf{F}}(\mathcal{Y}_g(\mathcal{Y}_0^C)))$. 

Denoting the result of running the greedy maximization algorithm \ref{al:greedy_max} on the proxy function $\tilde{G}(\mathcal{Y}_0)$ with $\tilde{\mathcal{Y}}_0^{\textbf{gr}}$ we have our final result:
\begin{theorem} Let the maximum covariance discrepancy between the true and approximate models be bounded according to $$\forall k_1,k_2,i_1,i_2: \left\lvert k(\textbf{y}^{(k_1)}(t_{i_1}),\textbf{y}^{(k_2)}(t_{i_2}))-k(\textbf{y}_G^{(k_1)}(t_{i_1}),\textbf{y}_G^{(k_2)}(t_{i_2})) \right\rvert\leq \delta$$ then we have 	$${G}(\tilde{\mathcal{Y}}_0^{\textbf{gr}})\geq (1-e^{-1})(G(\mathcal{Y}_0^{\star})+O(\log(1-\text{const}\cdot\delta)))$$
\end{theorem}
\begin{proof}
	Immediate from Lemma \ref{lem:nemhauser} and Theorem \ref{thm:1}.
\end{proof}
The last theorem demonstrates that applying the greedy maximization algorithm on the proxy function $\tilde{G}(\cdot)$ retrieves a solution $\tilde{\mathcal{Y}}_0^{\textbf{gr}}$ which is near optimal for the original function $G(\cdot)$, which is what we want.

\subsection{Leveraging Additional Techniques  in Submodular Optimization}
In this section we briefly survey additional results of interest from the literature on submodular maximization.
\subsubsection{Lazy Greedy Submodular Maximization}
The computational complexity of the greedy algorithm, while tractable in many settings, can be driven down further using the submodularity property of the set function. The so called Lazy greedy maximization algorithm (algorithm \ref{al:lazy_greedy_max}) which relies on the submodulairty of $G$ is often found to empirically decrease running time by orders of magnitude \cite{minoux1978accelerated}. Our numerical experiments of Section \ref{sec:numerical_experiments} utilize this algorithm for all relevant simulations.
	\begin{algorithm}
		\caption{Lazy Greedy Submodular Maximization}
		\label{al:lazy_greedy_max}
		\begin{algorithmic}[1]
			\STATE $\mathcal{S}\leftarrow \emptyset, \quad \forall x\in \mathcal{V} : m[x]\leftarrow \infty$
			\FOR{$i=1$ to $K$}
			\STATE $\text{STOP}\leftarrow 0$
			\WHILE {$\sim\text{STOP}$}
			\STATE $x^{\star}=\text{argmax}_{x\in \mathcal{V}\setminus \mathcal{S}}\,G(\mathcal{S}\cup \left\lbrace x \right\rbrace )$
			\STATE $m[x^\star]=G(\mathcal{S}\cup \left\lbrace x \right\rbrace )-G(\mathcal{S})$
			\IF{$m[x^\star]\geq \text{argmax}_{x}\;m[x] $}
			\STATE $\text{STOP}\leftarrow 1$
			\ENDIF
			\ENDWHILE
			\STATE $\mathcal{S} \leftarrow \mathcal{S} \cup \left\lbrace x^{\star}\right\rbrace $
			\ENDFOR
			\STATE Return $\mathcal{S}$
		\end{algorithmic}
	\end{algorithm}
\subsubsection{Submodular Maximization with Matroid Constraints}
We identified our approximated experimental design problem \eqref{eq:opt_prob_exp_des_app} as one of maximizing a submodular function under a cardinality constraint on a subset of $\mathcal{Y}$. With the argument identified as submodular we can define variants of the cardinality constrained problem that may be of interest in applications and retain the efficient approximation property of \eqref{eq:opt_prob_exp_des_app}.

We briefly mention submodular maximization with matroid constraints \cite{krause2012submodular},\cite{calinescu2007maximizing}, where in lieu of \eqref{eq:opt_prob_exp_des_app} we solve:
\begin{align} \label{eq:opt_prob_exp_des_app_matroid}
{Y}_0^\star= \underset{\mathcal{Y}_0: \mathcal{Y}_0\in\mathcal{I}}{\text{argmax}}\; \tilde{G}(\mathcal{Y}_0)
\end{align}
and $\mathcal{I}$ is a matroid combinatorial structure \cite{oxley2006matroid}. Matroids can concisely capture complicated constraints on $\mathcal{Y}_0$, for example let $\lbrace\mathcal{Y}^i \rbrace$ be a partition of $\mathcal{Y}$, i.e. $\bigcup_i \mathcal{Y}^i=\mathcal{Y}$, $\forall i\neq j: \mathcal{Y}^i\bigcap \mathcal{Y}^j=\emptyset$. With the partition in place, a constraint on $\mathcal{Y}_0$ of the form $\mathcal{Y}_0 \bigcap \mathcal{Y}^i \leq K_i$ can be shown to be a matroid constraint of the form $\mathcal{Y}_0\in\mathcal{I}$. A constraint like this is useful for designing experiments to learn misspecified models where we cannot choose more than a limited number $K_i$ of initial conditions to lie in any specific region $\mathcal{Y}^i$, e.g. due to some physical impediment for repeating experiments with similar conditions. It may be shown that an efficient greedy algorithm can approximate the optimal solution of problems such as those mentioned despite the exact problem being generally NP-hard.
	\section{Numerical Experiments} \label{sec:numerical_experiments}
In this section we discuss results of numerical experiments validating and demonstrating our techniques.
\subsection{Correction Term Fitting via GP Regression} \label{sec:num_gp}
For the first experiment we consider a misspecified system in $d=2$ dimensions where the known component is a fixed linear (matrix) operator, $\textbf{G}(\textbf{y}(t))=\textbf{A}\textbf{y}(t)$  with 
$$\textbf{A} = \left[ \begin{matrix} +0.02&+0.10\\ -0.10&-0.06 \end{matrix}\right], $$ and the misspecified component is set according to $\textbf{F}([y_1,y_2]^\top)=[0.01y_1^2, 0.01y_2^2]^\top$. We observe the system evolution over the time span $t\in [0,6]$, collecting $T=11$ equally-spaced time samples per experiment. The sampled time evolution sequences $\textbf{y}^{(k)}(t)$ were computed exactly, and we have measured noisy samples $\tilde{\textbf{F}}(\cdot)$ along the evolution path as per the observation model \eqref{eq:noisy_measurements}, where the measurement noise was taken as $\boldsymbol{\Sigma}_{\epsilon}=\sigma_{\epsilon}^2 \textbf{I}$ with $\sigma_{\epsilon}^2=10^{-4}$.

Figure \ref{fig:time_evolution} (left) depicts $K=40$ trajectories $\textbf{y}(t)$ (solid lines) induced by a set $\mathcal{Y}_0$  of initial conditions (black dots). Elements $\textbf{y}\in\mathcal{Y}_0$ were drawn from a  uniform distribution over the square $\mathcal{D}=[-1,+1]{\times}[-1,+1]$. For comparison, we overlay the corresponding trajectories of the misspecified model $\textbf{y}_G(t)$ taking into account solely the linear driving term $\textbf{G}(\cdot)$ (dashed lines).

\begin{figure*}
	\makebox[\textwidth][c]{\includegraphics[width=1.4\textwidth]{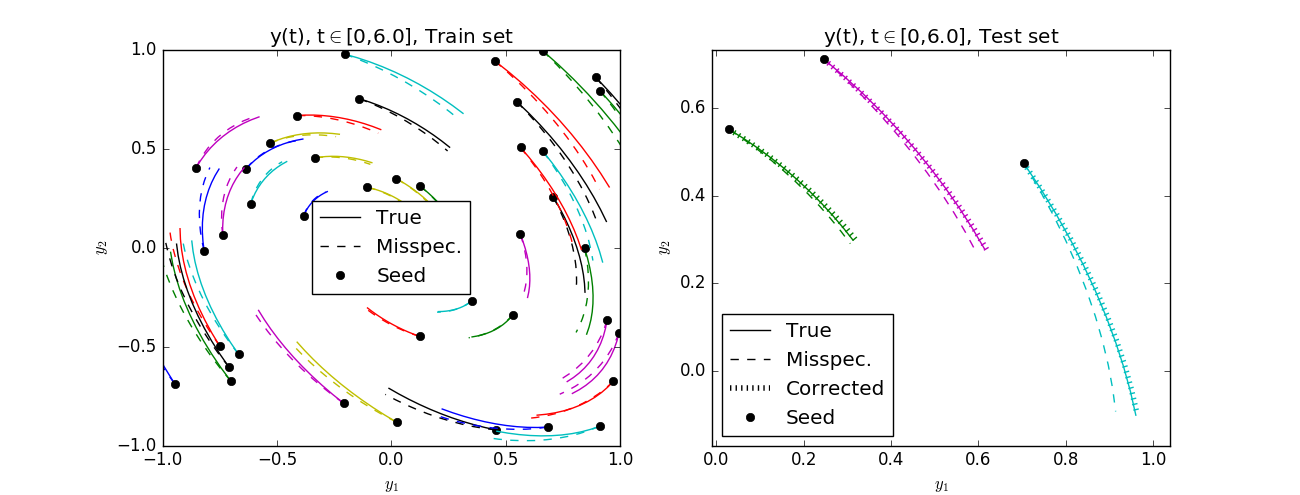}}
	\caption{Time evolution of system output. (left) Training set, with actual evolution in solid lines and misspecified predictions in dashed lines (right) Test set, with corrected predictions overlaid.} \label{fig:time_evolution}
\end{figure*}

For the GP regression we use a Gaussian kernel with $\sigma_w^2=1.0$ scaled for local variance $\frac{1}{\vert\mathcal{D}\vert}\iint_{\mathcal{D}}\vert F_{1}(\textbf{y})\vert^2d\textbf{y}=\frac{1}{\vert\mathcal{D}\vert}\iint_{\mathcal{D}}\vert F_{2}(\textbf{y})\vert^2d\textbf{y}=4\cdot 10^{-5}$. Figure \ref{fig:correction_term_error} depicts the estimation error $\|\hat{\textbf{F}}(\textbf{y})-\textbf{F}(\textbf{y})\|_2$ for $\textbf{y}\in \mathcal{D}$, overlaid with the training sequences. As is evident from these plots the estimation fidelity is high in the regions where training data is readily available.

Finally, in Figure \ref{fig:time_evolution} (right) the estimated correction term $\hat{\textbf{F}}(\cdot)$ was used to test prediction performance over some arbitrary set of initial conditions, and compare to the misspecified predicted evolution. The corrected curves (striped lines) are evidently closer to the true paths (solid lines) compared to the misspecified predictions (dashed lines).
\begin{figure*}
	\makebox[\textwidth][c]{\includegraphics[width=0.7\textwidth]{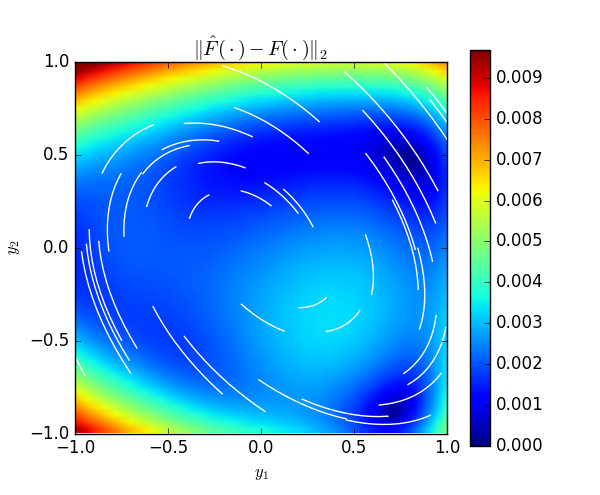}}
	\caption{Estimation error $\|\hat{\textbf{F}}(\cdot)-\textbf{F}(\cdot)\|_2$. White overlaid traces depict the training set time evolution sequences.} \label{fig:correction_term_error}
\end{figure*}

\subsection{Experimental Design for a Dynamical System}
In this subsection we experiment with and implement the experimental design procedures detailed in Section \ref{sec:experimental_design}. We are interested in designing a succession of $K=9$ experiments. The experimental design entails selecting an optimal set $\mathcal{Y}_0 \subseteq \mathcal{Y}$ of initial conditions from which to start the system off. With the misspecified system as defined in the previous subsection, we take the possible selection set $\mathcal{Y}$ to be a uniformly spaced two dimensional $13{\times} 13$ grid in $\mathcal{D}=[-1,+1]{\times}[-1,+1]$ as depicted in Figure \ref{fig:time_evolution_experimental} (left). We implement the lazy greedy algorithm and design an approximately optimal selection set $\mathcal{Y}_0$, marked with black squares in Figure \ref{fig:time_evolution_experimental} (left). Performance is compared to a seed of equal size chosen randomly over $\mathcal{Y}$ marked in black circles. Prediction performance over some arbitrary test set of initial conditions is presented in Figure \ref{fig:time_evolution_experimental} (right) and a heat map for the estimation error in $\hat{\textbf{F}}(\cdot)$ is plotted in Figure \ref{fig:correction_term_error_experimental}.

\begin{figure*}
\makebox[\textwidth][c]{\includegraphics[width=1.4\textwidth]{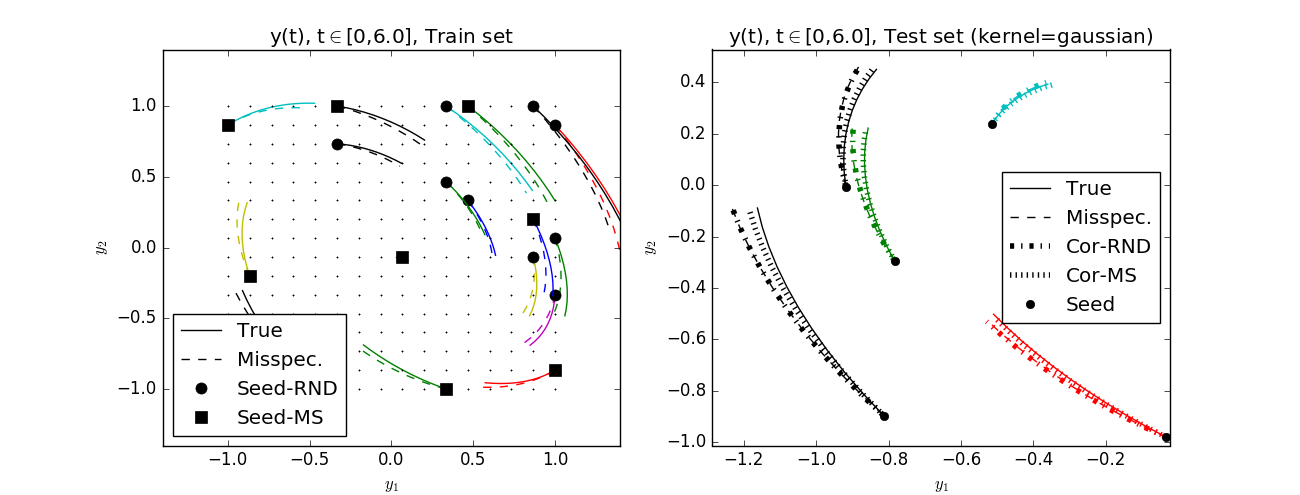}}
\caption{Experimental design setup. (left) Training data collected in two setups, first random and second based on designing experiments to match the misspecified dynamics. (right) example of prediction test on some arbitrary initial conditions.} \label{fig:time_evolution_experimental}
\end{figure*}

\begin{figure*}
\makebox[\textwidth][c]{\includegraphics[width=1\textwidth]{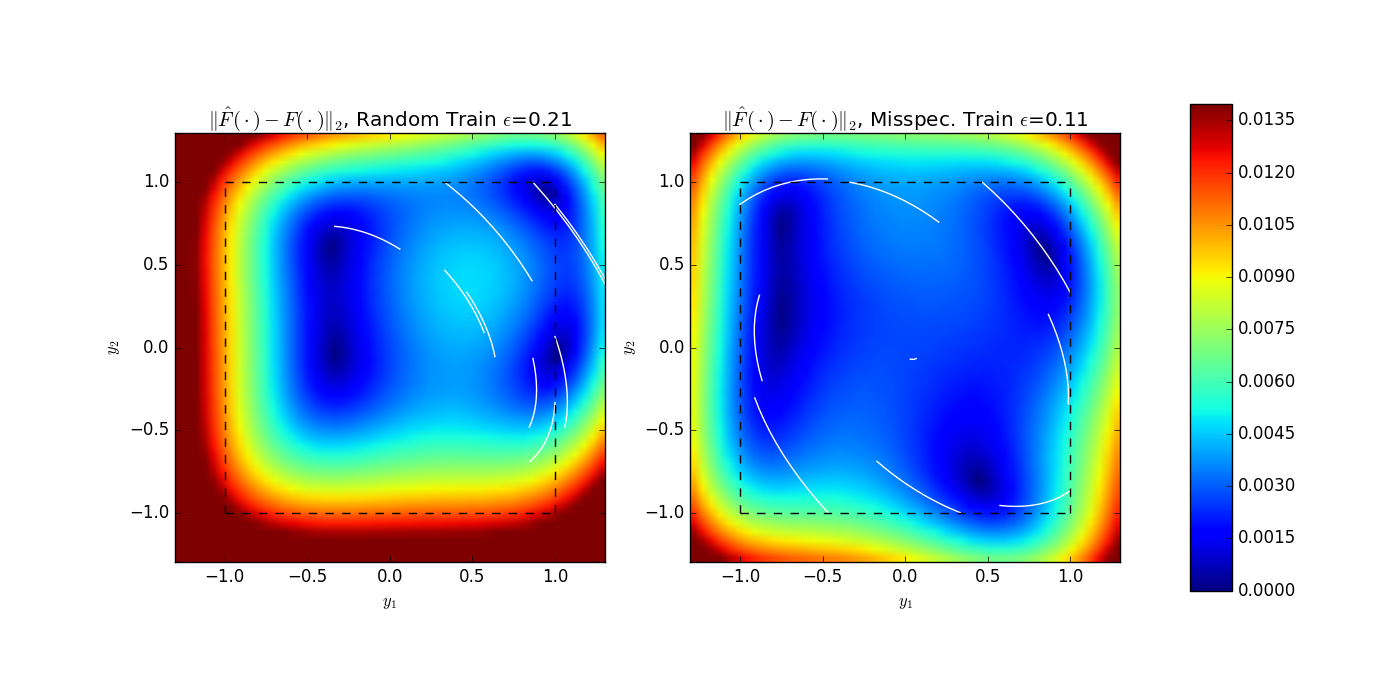}}
\caption{Absolute error in the correction term $\vert\hat{F}_i(\cdot)-{F}_i(\cdot)\vert$. (top) Random initial conditions (bottom) Experimental design.} \label{fig:correction_term_error_experimental}
\end{figure*}

Our next experiment involved changing the training set size, keeping track of estimation performance as measured according to $\iint_{\mathcal{D} }\|\hat{\textbf{F}}(\textbf{y})-\textbf{F}(\textbf{y})\|_2d\textbf{y}$ (estimated via numerical integration). Our dynamical system is as previously described, and we compare several correction strategies as summarized in Figure \ref{fig:correction_term_error}.

The first comparison is against a fully data driven estimator, which has no knowledge (not even approximate) of the system dynamics. We use training sequences as determined by our misspecified experimental design procedure but learn the full dynamics by applying GP regression with a Gaussian RBF kernel of scaled power $10^{-2}$ (due to the higher energy of the unknown function when the entire driving term is to be learned) and estimate the full system dynamics. The two other estimators are the ones previously described, namely estimating just the correction component using the knowledge about the approximate (misspecified) system dynamics, done once with a random seed training set and again with a training set seeded by a choice of initial conditions determined according to our misspecified experimental design procedure. 

The results are averaged over $10$ realizations of this setup. Also for comparison we show the energy of the correction term $\iint_{\mathcal{D}}\textbf{F}(\textbf{y})d \textbf{y}$ and the energy of the entire dynamics term $\iint_{\mathcal{D}}\left[\textbf{F}(\textbf{y})+\textbf{G}(\textbf{y})\right] d \textbf{y}$ which quantify the effective error associated with the misspecified and the completely unknowable models.


\begin{figure*}
	\makebox[\textwidth][c]{\includegraphics[width=1\textwidth]{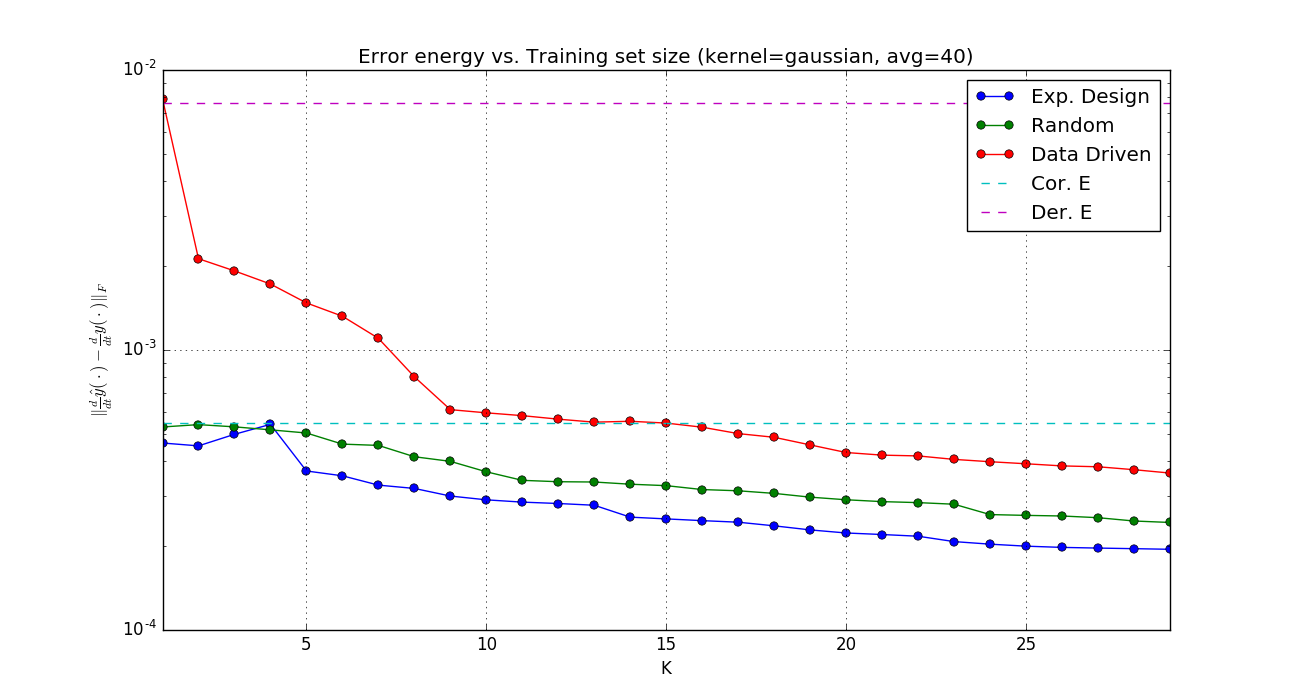}}
	\caption{Estimation error vs. training set size.} \label{fig:training_set_size}
\end{figure*}

Evidently the fully data driven approach is always the worst as it ignores the data embedded in the approximated model. However, with increasing number of experiments the difference between this approach and the ones taking into account the approximate dynamics tends to diminish, as the data becomes abundant and no prior assumptions about the model are needed. The approach taking into account the known component in designing the experimental setup is superior as it utilizes all available knowledge. The random training ignoring the known dynamics component incurs a cost in terms of estimation performance compared to the experimental design approach. 

\subsection{A Misspecified Gravitational Field}
Experimental design is crucial when the cost of experiments is high. One plausible scenario is in which a gravitational field
is estimated by controlled experiments of placing an object and observing its free fall (such experiments are likely to 
be costly). Accurate models of gravitational field can be useful in planning satellite trajectories around a planet. 
We use an artificial simplified simulation of the above in which we explore a problem of motion in a two-dimensional gravitational field.
If the gravitational field around the planet is fully characterized then this motion can be easily simulated through the laws of mechanics. In our setting we assume that the gravitational field is not fully known, in reality this could happen due to e.g. nonuniform mass distribution for the planet or gravitational influence from other nearby heavy masses \cite{muller1968mascons,rossi1999orbital,rosson2016orbital}.

Concretely, The two dimensional space is populated with a set of fixed objects, e.g. stars, with the $\textit{i}$th object having mass $m_i$ and position $\boldsymbol{x}^{i}$ and we are interested in solving for the motion of some free-moving unit mass, i.e. a satellite, in the corresponding gravitational field. Let $\boldsymbol{x}(t)=[x_1(t),x_2(t)]^{T}$ be the coordinate vector of the free-moving unit mass. The equations of motion governing the time evolution of $\boldsymbol{x}(t)$ are prescribed by classical mechanics and given according to \cite{goldstein1965classical}:
\begin{align}
\frac{d^2}{dt^2}\boldsymbol{x}(t)=-\sum\limits_{i}m_i\frac{\boldsymbol{x}(t)-\boldsymbol{x}^{i}}{\|\boldsymbol{x}(t)-\boldsymbol{x}^{i}\|^3}
\end{align}
This is a second order ODE expressing Newton's second law of motion and the gravitational field force. Namely, the acceleration experienced by the satellite is equal to the sum of forces acting on it. The force exerted on the satellite by the $\textit{i}$th mass is aligned with the vector connecting the two and is directly proportional to $m_i$ and inversely proportional to the squared distance between them.

The second order ODE may be converted into first order form by introducing new variables and defining the transformation  
\begin{equation}
[y_1(t),y_2(t),y_3(t),y_4(t)]^{\top}\equiv[x_1(t),x_2(t),\frac{d}{dt}x_1(t),\frac{d}{dt}x_2(t)]
\end{equation}

In the new variables the equations of motion read:
\begin{align}
\frac{d}{dt}y_1(t)&=y_3(t)\\
\frac{d}{dt}y_2(t)&=y_4(t)\\
\frac{d}{dt}y_3(t)&=-\sum\limits_{i}m_i\frac{y_1(t)-{x}_{1}^i}{\|[y_1(t),y_2(t)]^{\top}-\boldsymbol{x}^{i}\|^3}\\
\frac{d}{dt}y_4(t)&=-\sum\limits_{i}m_i\frac{y_2(t)-{x}_{2}^i}{\|[y_1(t),y_2(t)]^{\top}-\boldsymbol{x}^{i}\|^3}
\end{align}
which is a first order system of ODE as in \eqref{eq:mis_specified_model}. 

We consider a known but misspecified model that takes into account a single fixed mass in the origin with $m_1=0.2$ and $\boldsymbol{x}^1=[0,0]^\top$. The true model however includes two additional masses $m_2=0.1,m_3=0.4$ and $\boldsymbol{x}^2=[0,4]^\top,\boldsymbol{x}^3=[0.5,3.8]^\top$. With these symbols, we have
\begin{align}
\textbf{G}(\textbf{y}(t))&{=}\left[y_3(t),y_4(t),\frac{{-}m_1(y_1(t){-}{x}_{1}^1)}{\|[y_1(t),y_2(t)]^{\top}{-}\boldsymbol{x}^{1}\|^3},\frac{{-}m_1(y_2(t){-}{x}_{2}^1)}{\|[y_1(t),y_2(t)]^{\top}{-}\boldsymbol{x}^{1}\|^3}\right]^{\top}\\
\textbf{F}(\textbf{y}(t))&{=}\left[0,0,\sum\limits_{i=2,3}\frac{-m_i(y_1(t){-}{x}_{1}^i)}{\|[y_1(t),y_2(t)]^{\top}{-}\boldsymbol{x}^{i}\|^3},\sum\limits_{i=2,3}\frac{{-}m_i(y_2(t){-}{x}_{2}^i)}{\|[y_1(t),y_2(t)]^{\top}{-}\boldsymbol{x}^{i}\|^3}\right]^{\top}
\end{align}

For this experiment the signal $\textbf{y}(t)$ is 4 dimensional such that at any moment it captures the location as well as vector velocity of the satellite. Similarly, initial conditions are specified in this four dimensional space.
 
We limit our attention to correction functions of the functional form $\textbf{F}([y_1,y_2,y_3,y_4])=[0,0,\textbf{F}_{3,4}([y_1,y_2])]^{\top}$, i.e. the gravitational field correction is strictly a function of the spatial coordinates $(y_1,y_2)$, and has only two unknown components. We thus consider the problem of estimating $\textbf{F}_{3,4}:\mathbb{R}^2\rightarrow\mathbb{R}^2$, and our results and techniques naturally carry over to this scenario.
 
For the kernel we use a Gaussian RBF with $\sigma^2_k=1.0$ scaled for local variance ${10}^{-3}$ and the measurement noise is $\boldsymbol\Sigma_\epsilon=10^{-4}\textbf{I}$. Experiments run in the time frame $t\in[0,3.0]$ and $T=20$ data samples are collected per experiment. The selection set $\mathcal{Y}$ is a set of size $\vert\mathcal{Y}\vert=300$ of initial conditions, whose spatial coordinates $(y_1,y_2)$ are depicted in Figure \ref{fig:gravity_prediction} (left) in addition to the mass configuration in space.
Also shown are training sets of size $K=7$ as selected via an agnostic experimental design procedure and a misspecified aided one. In Figure \ref{fig:gravity_prediction} (right) we showcase prediction performance on a random test set. Both the agnostic and the misspecified designs perform well here compared to the misspecified predictions.
 
Figure \ref{fig:gravity_error} plots the estimation error of $\hat{\textbf{F}}_{3,4}(\cdot)$ for the setup above for the agnostic design (left) and the misspecified guided design (right) which performs slightly better when compared according to the mean squared error over the domain of interest $\mathcal{D}$ delineated inside the dashed line.
 
Finally in Figure \ref{fig:gravity_comparison} we compare the mean square error for the two methods as a function of $K$, as determined empirically by averaging the results of $400$ noise realizations. For reference, the dashed red line depicts the mean energy in the unknown term $\textbf{F}_{3,4}(\cdot)$.
\begin{figure*}
\makebox[\textwidth][c]{\includegraphics[width=1.4\textwidth]{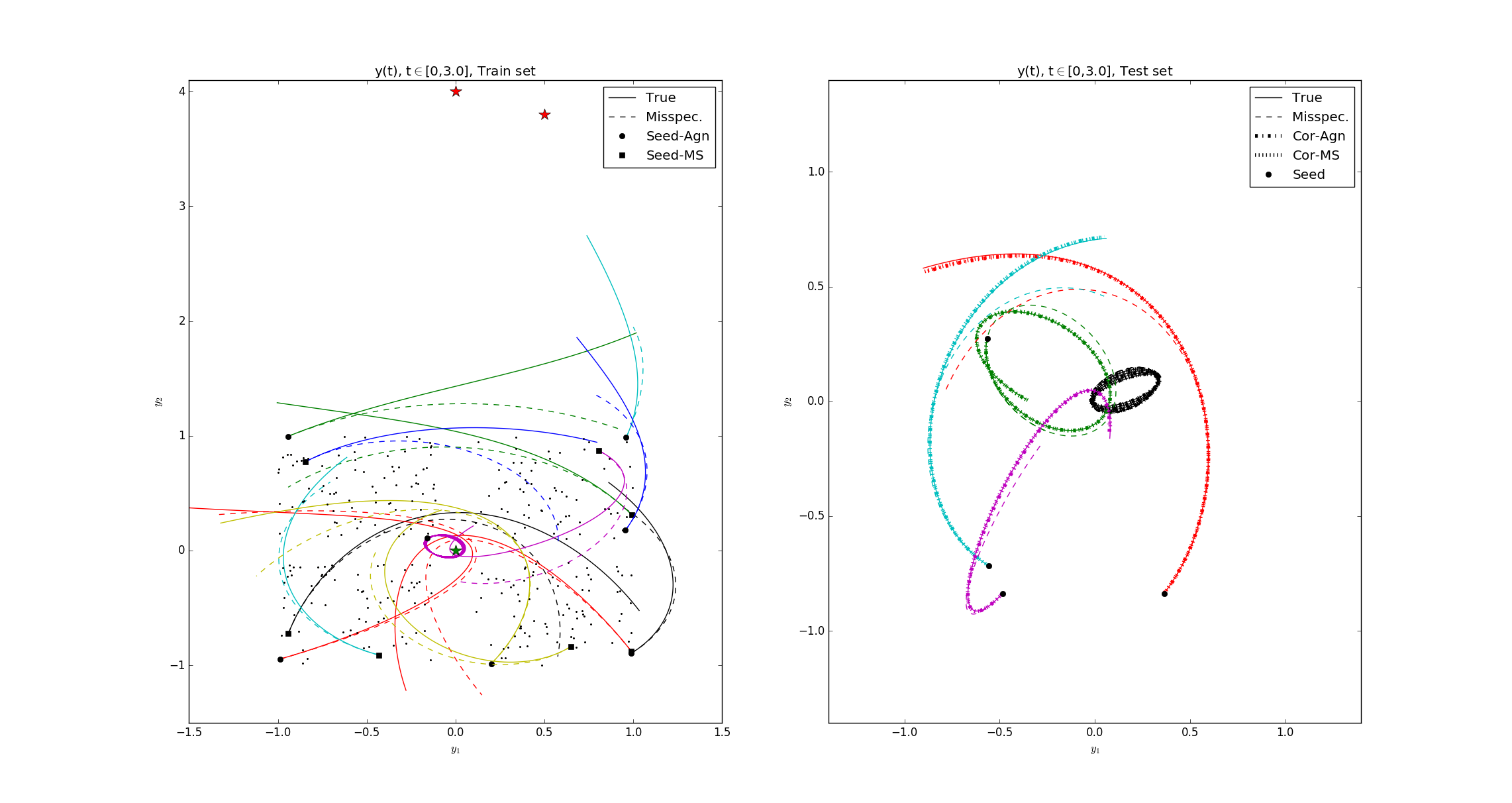}}
\caption{Experimental design in a misspecified gravitational field. Training set as determined via an agnostic approach and a misspecified aided approach (left) and prediction over a random test set (right).} \label{fig:gravity_prediction}
\end{figure*}
\begin{figure*}
\makebox[\textwidth][c]{\includegraphics[width=1\textwidth]{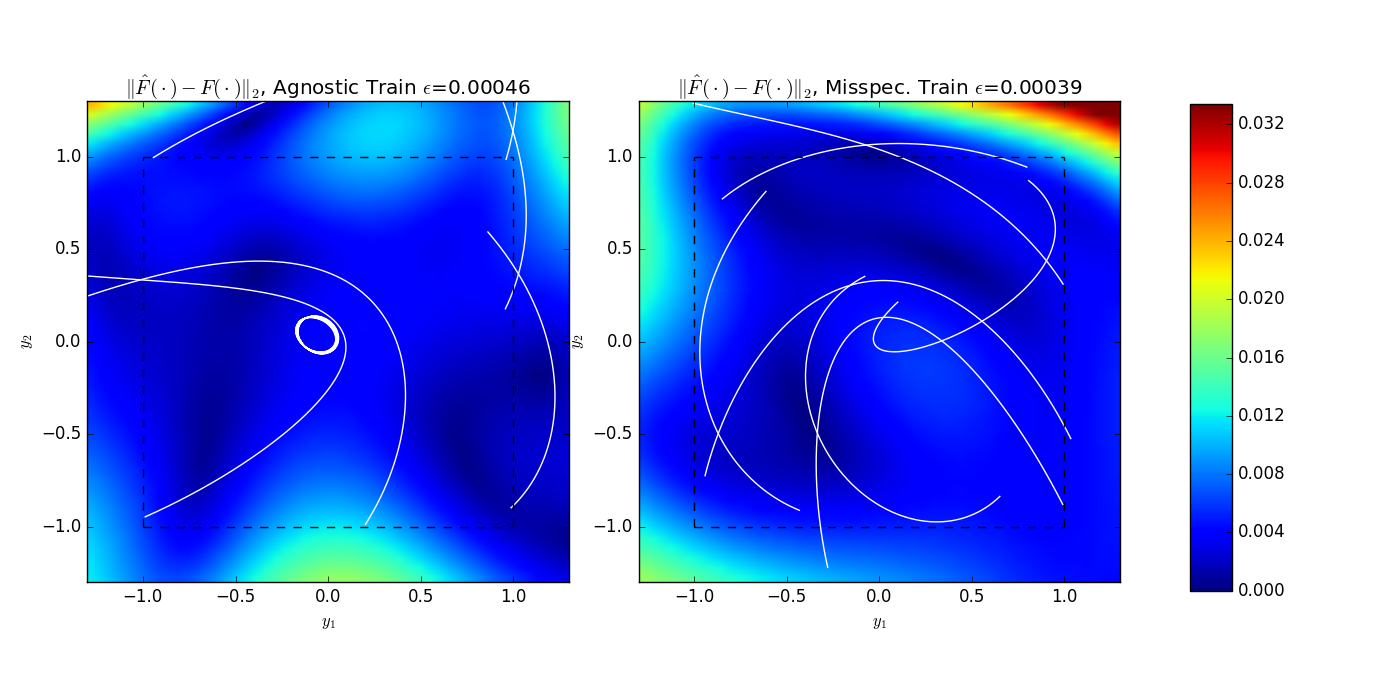}}
\caption{Estimation error map for $\hat{\textbf{F}}_{3,4}(\cdot)$ for an agnostic choice of training set (left) and a misspecified aided design (right)} \label{fig:gravity_error}
\end{figure*}
\begin{figure*}
\makebox[\textwidth][c]{\includegraphics[width=0.8\textwidth]{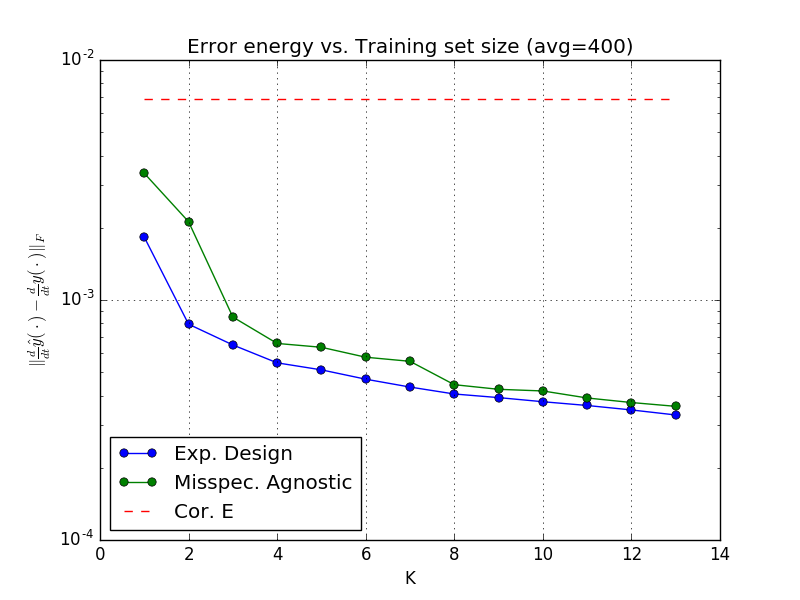}}
\caption{Average estimation error vs. training set size.} \label{fig:gravity_comparison}
\end{figure*}
	\section{Conclusions and Future Extensions}\label{sec:conclusions-and-future-extensions}
We have introduced a flexible Gaussian Process based formalism for expressing misspecified models for dynamical systems, and a corresponding technique for making inference and learning the misspecified dynamics based on empirical data collected from system evolution sequences. We formulated a corresponding optimal experimental design problem as one of choosing informative initial conditions that facilitate rapid learning of the system, and suggested an efficient algorithm with guarantees to find approximate such designs under an experimental budget constraint. 

Several aspects of our work may be extended. We leave the following ideas and directions for future research. In this study, we have assumed that empirical data is collected only after experimental design has been performed. However, in various configurations, it is possible to consider an online adaptive experimental design formulation, where sequential predictions are made based on past observations. While one can consider a setting in which the aforementioned design process is being re-executed following each observation (with updated knowledge), such approach may be sub-optimal. Recent studies have been considering approaches such as dynamic programming in the context of Bayesian optimization, to devise experimental design in a less myopic fashion \cite{poloczek2016multi,lam2016bayesian}. On another matter, in the current study, the design space involved a discrete lattice of prospective seed coordinates (initial conditions starting points). Alternative, spatially continuous parametrization of the seeding points, may be more appropriate in other circumstances, and may enable harnessing scalable, continuous optimization strategies for determination of the initial states. While we attempted to generalize the functional form of the correction model by the utilization of 
a Gaussian Process as a generic form of model correction, the overall relationship of the correction term to the misspecified model is still in the form of an additive term. This popular choice may be appropriate for a broad range of applications, but obviously, for others, more sophisticated forms should be considered. 

In this study, we have focused our attention at the link of sub-modularity and the mutual information measure. In future studies it would be beneficial to explore the relation between sub-modularity and other inference performance measures. Additionally we leave for future research full consideration of the measurement error in dynamical systems state variables for enhancing the modeling power of our formulations.
Lastly, from a computational standpoint, the incorporation of efficient, random features based  methods \cite{rahimi2007random,HaimSto} for accelerated predictions over the corrected system, would enable scalability of the approach towards complex large-scale problems.
 

	\bibliographystyle{siamplain}
	\bibliography{Dynamic_mis_model_bib}
\end{document}